\documentclass{article}

\usepackage[nonatbib,final]{neurips_2023}

\usepackage[utf8]{inputenc} % allow utf-8 input
\usepackage[T1]{fontenc}    % use 8-bit T1 fonts
\usepackage[hypertexnames=false]{hyperref}       % hyperlinks
\usepackage{url}            % simple URL typesetting
\usepackage{booktabs}       % professional-quality tables
\usepackage{amsfonts}       % blackboard math symbols
\usepackage{nicefrac}       % compact symbols for 1/2, etc.
\usepackage{microtype}      % microtypography
\usepackage{xcolor}         % colors

\usepackage{amsmath,amsthm}
\usepackage{algorithm} 
\usepackage{algpseudocode}
\usepackage{mathtools}
\usepackage{cleveref}
\usepackage{tikz}
\usepackage{ifthen}

\usepackage{thmtools,thm-restate}

\usepackage{macros}
\hypersetup{
    colorlinks,
    linkcolor={black},
    citecolor={black},
    urlcolor={black}
}

\usepackage{environments}

\usepackage{etoolbox}
\makeatletter
\patchcmd{\@bibitem}{\ignorespaces}{\label{bib:#1}\ignorespaces}{}{}
\makeatother

\newcommand{\mycite}[1]{[\ref{bib:#1}]}

\newcommand{\mytcite}[1]{\texorpdfstring{\cite{#1}}{\mycite{#1}}}

\title{On the Minimax Regret for Online Learning\\with Feedback Graphs}

\author{%
  Khaled Eldowa\thanks{Equal contribution.}\\
  Università degli Studi di Milano, Milan, Italy\\
  \texttt{khaled.eldowa@unimi.it}\\
\And
  Emmanuel Esposito\samethanks\\
  Università degli Studi di Milano, Milan, Italy\\
  \& Istituto Italiano di Tecnologia, Genoa, Italy\\
  \texttt{emmanuel@emmanuelesposito.it}\\
\And
  Tommaso Cesari\\
  University of Ottawa, Ottawa, Canada\\
  \texttt{tcesari@uottawa.ca}\\
\And
  Nicolò Cesa-Bianchi\\
  Università degli Studi di Milano, Milan, Italy\\
  \& Politecnico di Milano, Milan, Italy\\
  \texttt{nicolo.cesa-bianchi@unimi.it}\\
}

\begin{document}

\maketitle
\begin{abstract}
In this work, we improve on the upper and lower bounds for the regret of online learning with strongly observable undirected feedback graphs.
The best known upper bound for this problem is $\mathcal{O}\bigl(\sqrt{\alpha T\ln K}\bigr)$, where $K$ is the number of actions, $\alpha$ is the independence number of the graph, and $T$ is the time horizon. The $\sqrt{\ln K}$ factor is known to be necessary when $\alpha = 1$ (the experts case). On the other hand, when $\alpha = K$ (the bandits case), the minimax rate is known to be $\Theta\bigl(\sqrt{KT}\bigr)$, and a lower bound $\Omega\bigl(\sqrt{\alpha T}\bigr)$ is known to hold for any $\alpha$. Our improved upper bound $\mathcal{O}\bigl(\sqrt{\alpha T(1+\ln(K/\alpha))}\bigr)$ holds for any $\alpha$ and matches the lower bounds for bandits and experts, while interpolating intermediate cases. To prove this result, we use FTRL with $q$-Tsallis entropy for a carefully chosen value of $q \in [1/2, 1)$ that varies with $\alpha$. The analysis of this algorithm requires a new bound on the variance term in the regret.
We also show how to extend our techniques to time-varying graphs, without requiring prior knowledge of their independence numbers.
Our upper bound is complemented by an improved $\Omega\bigl(\sqrt{\alpha T(\ln K)/(\ln\alpha)}\bigr)$ lower bound for all $\alpha > 1$, whose analysis relies on a novel reduction to multitask learning.
This shows that a logarithmic factor is necessary as soon as $\alpha < K$.
\end{abstract}

\section{Introduction}
Feedback graphs \cite{mannor2011side} provide an elegant interpolation between two popular online learning models: multiarmed bandits and prediction with expert advice. When learning with an undirected feedback graph $G$ over $K$ actions, the online algorithm observes not only the loss of the action chosen in each round, but also the loss of the actions that are adjacent to it in the graph. Two important special cases of this setting are: prediction with expert advice (when $G$ is a clique) and $K$-armed bandits (when $G$ has no edges).
When losses are generated adversarially, the regret in the feedback graph setting with strong observability has been shown to scale with the independence number $\alpha$ of $G$. Intuitively, denser graphs, which correspond to smaller independence numbers, provide more feedback to the learner, thus enabling a better control on regret. More specifically, the best known upper and lower bounds on the regret after $T$ rounds are $\mathcal{O}\big(\sqrt{\alpha T\log K}\big)$ and $\Omega\big(\sqrt{\alpha T}\big)$ \cite{alon-journal,alon2013bandits}. 
It has been known for three decades that this upper bound is tight for $\alpha = 1$ (the experts case, \cite{cesa1997use,cesa1993use}). 
When $\alpha = K$ (the bandits case), the lower bound $\Omega\big(\sqrt{KT}\big)$---which has also been known for nearly three decades \cite{auer1995gambling,auer2002nonstochastic}---was matched by a corresponding upper bound $\mathcal{O}\big(\sqrt{KT}\big)$ only in 2009 \cite{audibert2009minimax}. These results show that in feedback graphs, the logarithmic factor $\sqrt{\log K}$ is necessary (at least) for the $\alpha=1$ case, while it must vanish from the minimax regret as $\alpha$ grows from $1$ to $K$, but the current bounds fail to capture this fact. In this work, we prove new upper and lower regret bounds that for the first time account for this vanishing logarithmic factor.

To prove our new upper bound, we use the standard FTRL algorithm run with the $q$-Tsallis entropy regularizer ($q$-FTRL for short). It is well-known \cite{abernethy2015fighting} that for $q=\frac{1}{2}$ this algorithm (run with appropriate loss estimates) achieves regret $\mathcal{O}\big(\sqrt{KT}\big)$ when $\alpha=K$ (bandits case), while for $q \to 1^-$ the same algorithm (without loss estimates) recovers the bound $\mathcal{O}\big(\sqrt{T\log K}\big)$ when $\alpha=1$ (experts case). 
When $G$ contains all self-loops, we show in \Cref{thm:regret-self-loops} that, if $q$ is chosen as a certain function $q(\alpha,K)$, then $q(\alpha,K)$-FTRL, run with standard importance-weighted loss estimates, achieves regret $\calO\bigl(\sqrt{\alpha T(1+\log(K/\alpha))}\bigr)$. This is a strict improvement over the previous bound, and matches the lower bounds for bandits and experts while interpolating the intermediate cases. This interpolation is reflected by our choice of $q$, which goes from $\frac{1}{2}$ to $1$ as $\alpha$ ranges from $1$ to $K$. The main technical hurdle in proving this result is an extension to arbitrary values of $q \in \big[\frac{1}{2},1\big)$ of a standard result---see, e.g., \cite[Lemma 3]{mannor2011side}---that bounds in terms of $\alpha$ the variance term in the regret of $q$-FTRL.
In \Cref{thm:regret}, using a modified loss estimate, this result is extended to any strongly observable undirected graph \cite{alon2015beyond}, a class of feedback graphs in which some of the actions do not reveal their loss when played. In \Cref{thm:regret-doubling}, we show via a doubling trick that our new upper bound can also be obtained (up to constant factors) without the need of knowing (or computing) $\alpha$. As the resulting algorithm is oblivious to $\alpha$, our analysis also applies to arbitrary sequences of graphs $G_t$, where $K$ is constant but the independence number $\alpha_t$ of $G_t$ can change over time, and the algorithm observes $G_t$ only after choosing an action (the so-called uninformed case). In this setting, the analysis of the doubling trick is complicated by the non-trivial dependence of the regret on the sequence of $\alpha_t$.

We also improve on the $\Omega\big(\sqrt{\alpha T}\big)$ lower bound by proving a new $\Omega\bigl(\sqrt{\alpha T\log_{\alpha} K}\bigr)$ lower bound for all $\alpha > 1$. This is the first result showing the necessity---outside the experts case---of a logarithmic factor in the minimax regret for all $\alpha < K$. Our proof uses a stochastic adversary generating both losses and feedback graphs via i.i.d.\ draws from a joint distribution. This sequence of losses and feedback graphs can be used to define a hard instance of the multi-task bandits problem, a variant of the combinatorial bandits framework \cite{cesa2012combinatorial}. We then prove our result by adapting known lower bounding techniques for multi-task bandits \cite{audibert2014combinatorial}.
Note that for values of $\alpha$ bounded away from $2$ and $K$, the logarithmic factor $\log_{\alpha} K$ in the lower bound is smaller than the corresponding factor $1 + \log(K/\alpha)$ in the upper bound. Closing this gap remains an open problem.

\subsection{Additional related work}
Several previous works have used the $q$-Tsallis regularizer with $q$ tuned to specific values other than $\frac{1}{2}$ and $1$.
For example, in \cite[Section~4]{zimmert2019connections}, $q$ is chosen as a function of $K$ to prove a regret bound of $\calO\big(\sqrt{\alpha T(\log K)^3}\big)$ for any strongly observable directed feedback graph, which shaves off a $\log T$ factor compared to previous works. This bound is worse than the corresponding bounds for undirected graphs because the directed setting is harder.
Specific choices of $q$ have been considered to improve the regret in settings of online learning with standard bandit feedback. For example, the choice $q = \frac{2}{3}$ was used in \cite{rouyer2020tsallis} to improve the analysis of regret in bandits with decoupled exploration and exploitation.
Regret bounds for arbitrary choices of $q$ are derived in \cite{zimmert2021tsallis,jin2023improved} for a best-of-both-worlds analysis of bandits, though $q=\frac12$ remains the optimal choice.
The $\frac12$-Tsallis entropy and the Shannon entropy ($q=1$) regularizers have been combined before in different ways to obtain best-of-both-worlds guarantees for the graph feedback problem \cite{erez2021towards,ito2022bobw}.
The idea of using values of $q \in (\frac12,1)$ for feedback graphs is quite natural and has been brought up before, e.g., in~\cite{rouyer2022}, but achieving an improved dependence on the graph structure by picking a suitable value of~$q$ has not been, to the best of our knowledge, successfully pursued before.
On the other hand, an approach based on a similar use of the $q$-Tsallis regularizer has been employed by \cite{kwon2016gains} for the problem of multiarmed bandits with sparse losses to achieve a $\calO\brb{\sqrt{sT\ln(K/s)}}$ regret bound, where $s$ is the maximum number of nonzero losses at any round.

Our lower bound is reminiscent of the $\Omega\bigl(\sqrt{K T\log_{K} N}\bigr)$ lower bound proved in \cite{seldin2016lower} for the problem of bandits with expert advice (with $N \geq K$ being the number of experts); see also \cite{eldowa2023information} and \cite{vural2019minimax}.
In that problem, at each time step, experts suggest distributions over actions to the learner, whose regret is computed against the best expert in hindsight. Although the two settings are different, the variant of the multitask bandit problem that our lower bound construction simulates is the same as the one used in the proof of \cite[Theorem~7]{eldowa2023information}.

\section{Problem Setting}
For any integer $n \ge 1$, let $[n] = \{1,\dots,n\}$.
We consider the following game played over $T$ rounds between a learner with action set $V = [K]$ and the environment.
At the beginning of the game, the environment secretly selects a sequence of losses $(\loss_t)_{t\in[T]}$, where $\loss_t\colon V \to [0,1]$,\footnote{For notational convenience, we will sometimes treat the loss functions $\loss_t \colon V \to [0,1]$ as vectors with components $\loss_t(1),\ldots,\loss_t(K)$.} and a sequence of undirected graphs $(G_t)_{t\in[T]}$ over the set of actions $V$, that is, $G_t = (V,E_t)$.
At any time $t$, the learner selects an arm $I_t$ (possibly at random), then pays loss $\loss_t(I_t)$ and observes the feedback graph $G_t$ and all losses $\loss_t(i)$ of neighbouring actions $i\in \neigh{G_t}{I_t}$, where $\neigh{G_t}{i} = \{j\in V \,:\, (i,j)\in E_t\}$ (see Online Protocol \ref{protocol}).
In this work, we only focus on strongly observable graphs \cite{alon2015beyond}. An undirected graph $G$ is strongly observable if for every $i \in V$, at least one of the following holds: $i \in \neigh{G}{i}$ or $i \in \neigh{G}{j}$ for all $j \neq i$.

The performance of the learner is measured by the regret
\[
    R_T =
    \E \lsb{ \sum_{t=1}^T \loss_t(I_t) } - \min_{i\in[K]} \sum_{t=1}^T \loss_t(i) \enspace.
\]
where the expectation is over the learner's internal randomization.

{
\makeatletter
\renewcommand{\ALG@name}{Online Protocol}
\makeatother
\begin{algorithm}
\caption{Online learning with feedback graphs} \label{protocol}
    \begin{algorithmic}
        \State \textbf{environment:} (hidden) losses $\loss_t\colon V \to [0,1]$ and graphs $G_t = (V,E_t)$, for all $t = 1,\dots,T$
        \For{$t = 1, \dots$, $T$}
            \State The learner picks an action $I_t \in V$ (possibly at random)
            \State The learner incurs loss $\loss_t(I_t)$
            \State The learner observes losses $\bcb{ \brb{i,\loss_t(i)} : i \in \neigh{G_t}{I_t} }$ and graph $G_t$
        \EndFor
    \end{algorithmic}
\addtocounter{algorithm}{-1}
\end{algorithm}
}

We denote by $\Delta_K$ the simplex $\bcb{ p \in [0,1]^K : \lno{p}_1 = 1}$.
For any graph $G$, we define its independence number as the cardinality of the largest set of nodes such that no two nodes are neighbors, and denote it by $\alpha(G)$.
For simplicity, we use $N_t$ to denote the neighbourhood $N_{G_t}$ in the graph $G_t$ and we use $\alpha_t$ to denote the independence number $\alpha(G_t)$ of $G_t$ at time $t$.

\section{FTRL with Tsallis Entropy for Undirected Feedback Graphs} \label{s:main}

As a building block, in this section, we focus on the case when all the feedback graphs $G_1, \dots, G_T$ have the same independence number $\alpha_1 = \dots = \alpha_T = \alpha$, whereas the general case is treated in the next section.
For simplicity, we start with the assumption that all nodes have self-loops: $(i,i) \in E_t$ for all $i\in V$ and all $t$.
We later lift this requirement and show that the regret guarantees that we provide can be extended to general strongly observable undirected feedback graphs, only at the cost of a constant multiplicative factor.

The algorithm we analyze is $q$-FTRL (described in \Cref{alg:FTRL}), which is an instance of the follow the regularized leader (FTRL) framework---see, e.g., \cite[Chapter 7]{orabona2019modern}---with the (negative) $q$-Tsallis entropy 
\[
    \ftrlreg_\tparam(x) = \frac{1}{1-\tparam} \left(1 - \sum_{i \in V} x(i)^\tparam\right) \qquad \forall x \in \Delta_K \enspace,
\]
as the regularizer, whose parameter $q \in (0,1)$ can be tuned according to our needs.
Since we do not observe all the losses in a given round, the algorithm makes use of unbiased estimates for the losses.
When all self-loops are present, we define the estimated losses in the following standard manner.
Let $I_t$ be the action picked at round $t$, which is drawn from the distribution $p_t \in \Delta_K$ maintained by the algorithm, the loss estimate for an action $i \in V$ at round $t$ is given by
\begin{equation} \label{eq:unbiased-estimator-1}
    \hat{\loss}_t(i) = \frac{\loss_t(i)}{P_t(i)} \I\lrc{I_t \in \neigh{t}{i}} \enspace,
\end{equation}
where $P_t(i) = \mathbb{P}\brb{I_t \in \neigh{t}{i}} = \sum_{j \in \neigh{t}{i}} p_t(j)$.
This estimate is unbiased in the sense that $\E_t\bsb{\hat{\loss}_t(i)} = \loss_t(i)$ for all $t \in [T]$ and all $i \in V$, where we denote $\E_t\lsb{\cdot} = \E\lsb{\cdot \,|\, I_1, \dots, I_{t-1}}$.

\begin{algorithm}[ht]
    \caption{$\tparam$-FTRL for undirected feedback graphs} \label{alg:FTRL}
    \begin{algorithmic}
        \State \textbf{input:} $\tparam \in (0,1)$, $\eta > 0$
        \State \textbf{initialization:} $p_1(i) \gets 1/K$ for all $i=1,\ldots,K$
        \For{$t = 1, \ldots, T$}
            \State Select action $I_t \sim p_t$ and incur loss $\loss_{t}(I_t)$
            \State Observe losses $\bcb{\brb{i, \loss_{t}(i)} : i \in \neigh{t}{I_t}}$ and graph $G_t$
            \State Construct a loss estimate $\hat{\loss}_{t}$ for $\ell_t$ \Comment{e.g., \eqref{eq:unbiased-estimator-1} or~\eqref{eq:unbiased-estimator-2}}
            \State Let $p_{t+1} \gets \argmin_{p \in \Delta_{K}} \eta\ban{\sum_{s=1}^{t} \hat{\loss}_s, p} + \ftrlreg_\tparam(p)$
        \EndFor
    \end{algorithmic}
\end{algorithm}

A key part of the standard regret analysis of $\tparam$-FTRL (see, e.g., the proof of \Cref{lem:FTRL-Tsallis-bound} in \Cref{app:auxiliary}) is handling the variance term, which, with the choice of estimator given in \eqref{eq:unbiased-estimator-1}, takes the following form
\begin{equation}
    \tvar_t(\tparam) = \sum_{i \in V} \frac{p_t(i)^{2-\tparam}}{P_t(i)} \enspace.
\end{equation}
By H\"older's inequality, this term can be immediately upper bounded by
\[
    \tvar_t(q) \le \sum_{i \in V} p_t(i)^{1-\tparam} \le \Bbrb{\sum_{i \in V} p_t(i)}^{\!1-\tparam} \Bbrb{\sum_{i \in V} 1^{1/\tparam}}^{\!\tparam} = K^\tparam \enspace,
\]
while previous results on the regret analysis of multiarmed bandits with graph feedback~\cite{mannor2011side,alon-journal} would give
\[
    \tvar_t(q) \le \sum_{i \in V} \frac{p_t(i)}{P_t(i)} \le \alpha \enspace.
\]
However, the former result would only recover a $\calO(\sqrt{KT})$ regret bound (regardless of $\alpha$) with the best choice of $q=1/2$, 
which could be trivially achieved by ignoring side-observations of the losses, whereas the latter bound would only manage to achieve a $\calO(\sqrt{\alpha T\ln K})$ regret bound, incurring the extra $\sqrt{\ln K}$ factor for all values of $\alpha$.
Other results in the literature (e.g., see~\cite{alon2015beyond,alon2013bandits,dann2023blackbox,esposito2022learning,ito2022bobw,kocak2014efficient,rouyer2022,zimmert2019connections}) do not bring an improvement in this setting when bounding the $\tvar_t(q)$ term and, hence, do not suffice for achieving the desired regret bound.
The following lemma provides a novel and improved bound on quantities of the same form as $\tvar_t(q)$ in terms of the independence number $\alpha_t = \alpha$ of the undirected graph $G_t$.

\begin{lemma}\label{lem:turan-tsallis}
    Let $G = (V,E)$ be any undirected graph with $\abs{V}=K$ vertices and independence number $\alpha(G) = \alpha$.
    Let $b \in [0,1]$, $p \in \Delta_K$ and consider any nonempty subset $U \subseteq \lrc{v \in V : v \in \neigh{G}{v}}$.
    Then,
    \[
        \sum_{v \in U} \frac{p(v)^{1+b}}{\sum_{u \in \neigh{G}{v}} p(u)} \le \alpha^{1-b} \enspace.
    \]
\end{lemma}
\begin{proof}
    First of all, observe that we can restrict ourselves to the subgraph $G[U]$ induced by $U$, i.e., the graph $G[U] = (U, E \cap (U\times U))$.
    This is because the neighbourhoods in this graph are such that $\neigh{G[U]}{v} \subseteq \neigh{G}{v}$ for all $v \in U$, and its independence number is $\alpha(G[U]) \le \alpha(G)$.
    Hence, it suffices to prove the claimed inequality for any undirected graph $G = (V,E)$ with all self-loops, any $p \in [0,1]^K$ such that $\norm{p}_1 \le 1$, and the choice $U=V$.
    We assume this in what follows without loss of generality.
    
    For any subgraph $H \subseteq G$ with vertices $V(H) \subseteq V$, denote the quantity we want to upper bound by
    \[
        Q(H) = \sum_{v \in V(H)} \frac{p(v)^{1+b}}{\sum_{u \in \neigh{G}{v}} p(u)} \enspace.
    \]
    Our aim is thus to provide an upper bound to $Q(G)$.
    
    Consider a greedy algorithm that incrementally constructs a subset of vertices in the following way: at each step, it selects a vertex $v$ that maximizes $p(v)^b/\brb{\sum_{u \in \neigh{G}{v}} p(u)}$, it adds $v$ to the solution, and it removes $v$ from $G$ together with its neighbourhood $\neigh{G}{v}$.
    This step is iterated on the remaining graph until no vertex is left.

    Let $S = \lrc{v_1, \dots, v_s} \subseteq V$ be the solution returned by the above greedy algorithm on $G$.
    Also let $G_1, \dots, G_{s+1}$ be the sequence of graphs induced by the operations of the algorithm, where $G_1 = G$ and $G_{s+1}$ is the empty graph, and let $\neigh{r}{v} = \neigh{G_r}{v}$ for $v \in V(G_r)$.
    At every step $r \in [s]$ of the greedy algorithm, the contribution to $Q(G)$ of the removed vertices $\neigh{r}{v_r}$ amounts to
    \begin{align*}
        Q(G_r) - Q(G_{r+1})
        = \sum_{v \in \neigh{r}{v_r}} \frac{p(v)^{1+b}}{\sum_{u \in \neigh{1}{v}} p(u)}
        &\le \sum_{v \in \neigh{r}{v_r}} p(v) \frac{p(v_r)^b}{\sum_{u \in \neigh{1}{v_r}} p(u)} \\
        &\le \frac{\sum_{v \in \neigh{1}{v_r}} p(v)}{\sum_{u \in \neigh{1}{v_r}} p(u)} p(v_r)^b = p(v_r)^b \enspace,
    \end{align*}
    where the last inequality is due to the fact that $\neigh{i}{v} \subseteq \neigh{j}{v}$ for all $i\ge j$ and $v\in V_i$.
    Therefore, we can observe that
    \[
        Q(G) = \sum_{r=1}^s \brb{Q(G_r) - Q(G_{r+1})} \le \sum_{v \in S} p(v)^b \enspace.
    \]

    The solution $S$ is an independent set of $G$ by construction.
    Consider now any independent set $A \subseteq V$ of $G$.
    We have that
    \begin{align}
        \sum_{v \in A} p(v)^b
        &\le \max_{x \in \Delta_K} \sum_{v \in A} x(v)^b
        = \abs{A} \max_{x \in \Delta_K} \sum_{v \in A} \frac{x(v)^b}{\abs{A}} \nonumber\\
        &\le \abs{A} \max_{x \in \Delta_K} \bbrb{\frac{1}{\abs{A}}\sum_{v \in A} x(v)}^{\!b}
        \le \abs{A}^{1-b} \le \alpha^{1-b} \enspace,
    \end{align}
    where the second inequality follows by Jensen's inequality and the fact that $b \in [0,1]$.
\end{proof}

Observe that this upper bound is tight for general probability distributions $p \in \Delta_K$ over the vertices $V$ of any strongly observable undirected graph $G$ (containing at least one self-loop), as it is exactly achieved by the distribution $p^\star \in \Delta_K$ defined as $p^\star(i) = \frac{1}{\abs{S}} \I\lcb{i \in S}$ for some maximum independent set $S \subseteq V$ of $G$. Using this lemma, the following theorem provides our improved upper bound under the simplifying assumptions we made thus far.

\begin{theorem}\label{thm:regret-self-loops}
    Let $G_1, \dots, G_T$ be a sequence of undirected feedback graphs, where each $G_t$ contains all self-loops and has independence number $\alpha_t = \alpha$ for some common value $\alpha \in [K]$.
    If \Cref{alg:FTRL} is run with input
    \[
        \tparam = \frac12\bbrb{1 + \frac{\ln(K/\alpha)}{\sqrt{\ln(K/\alpha)^2+4} + 2}} \in [1/2,1)
        \qquad \text{ and } \qquad
        \eta = \sqrt{\frac{2\tparam K^{1-\tparam}}{T(1-\tparam)\alpha^\tparam}} \enspace,
    \]
    and loss estimates~\eqref{eq:unbiased-estimator-1}, then its regret satisfies $R_T \le 2\sqrt{e\alpha T \lrb{2 + \ln(K/\alpha)}}$
\end{theorem}
\begin{proof}
    One can verify that for any $i \in V$, the loss estimate $\hat{\loss}_t(i)$ defined in~\eqref{eq:unbiased-estimator-1} satisfies $\E_t\bsb{\hat{\loss}_t(i)^2} \leq 1/P_t(i)$. Hence, using also that $\E_t\bsb{\hat{\loss}_t(i)} = \loss_t(i)$, \Cref{lem:FTRL-Tsallis-bandits-bound} in \Cref{app:auxiliary} implies that
    \begin{align}
        R_T 
        &\le \frac{K^{1-\tparam}}{\eta(1-\tparam)} + \frac{\eta}{2\tparam} \sum_{t=1}^T \E\Bbsb{\sum_{i \in V} \frac{p_t(i)^{2-\tparam}}{P_t(i)}} \label{eq:FTRL-regret-with-variance-term} \\
        &\le \frac{K^{1-\tparam}}{\eta(1-\tparam)} + \frac{\eta}{2\tparam} \alpha^{\tparam}T \label{eq:FTRL-regret-bound-1} \enspace,
    \end{align}
    where the second inequality follows by \Cref{lem:turan-tsallis} with $b=1-\tparam$ since all actions $i \in V$ are such that $i \in \neigh{G}{i}$.
    Our choices for $\tparam$ and $\eta$ allow us to further upper bound the right-hand side of \eqref{eq:FTRL-regret-bound-1} by
    \begin{align*}
        \sqrt{\frac{2K^{1-\tparam}\alpha^{\tparam}}{q(1-q)}T}
        &= \sqrt{2T \exp\bbrb{1+\frac12 \ln\lr{\alpha K} - \frac12 \sqrt{\ln\lrb{K/\alpha}^2 + 4}} \bbrb{2 + \sqrt{\ln\lrb{K/\alpha}^2+4}}} \\
        &\le \sqrt{2e\alpha T \bbrb{2 + \sqrt{\ln\lrb{K/\alpha}^2+4}}}
        \le 2\sqrt{e\alpha T \sqrt{\ln\lrb{K/\alpha}^2+4}} \\
        &\le 2\sqrt{e\alpha T \lrb{2+\ln(K/\alpha)}}
        \enspace. \qedhere
    \end{align*}
\end{proof}

The regret bound achieved in the above theorem achieves the optimal regret bound for the experts setting (i.e., $\alpha=1$) and the bandits setting (i.e., $\alpha=K$) simultaneously. Moreover, it interpolates the intermediate cases for $\alpha$ ranging between $1$ and $K$, introducing the multiplicative logarithmic factor only for graphs with independence number strictly smaller than $K$. We remark that the chosen values of $q$ and $\eta$ do in fact minimize the right-hand side of \eqref{eq:FTRL-regret-bound-1}. 
Note that we relied on the knowledge of $\alpha$ to tune the parameter $q$. 
This is undesirable in general. We will show how to lift this requirement in \Cref{s:doubling}. The same comment applies to \Cref{thm:regret}, below.

We now show how to achieve the improved regret bound of \Cref{thm:regret-self-loops} in the case of strongly observable undirected feedback graphs where some self-loops may be missing; i.e., there may be actions $i \in V$ such that $i \notin \neigh{G}{i}$.
Using the loss estimator defined in \eqref{eq:unbiased-estimator-1} may lead to a large variance term due to the presence of actions without self-loops.
One approach to deal with this---see, e.g., \cite{zimmert2019connections} or \cite{luo2023highprobability}---is to suitably alter the loss estimates of these actions.

Define $S_t = \lcb{i \in V : i \notin \neigh{t}{i}}$ as the subset of actions without self-loops in the feedback graph $G_t$ at each time step $t \in [T]$.
The idea is that we need to carefully handle some action $i \in S_t$ only in the case when the probability $p_t(i)$ of choosing $i$ at round $t$ is sufficiently large, say, larger than $1/2$.
Define the set of such actions as $J_t = \lrc{i \in S_t : p_t(i) > 1/2}$ and observe that $\abs{J_t} \le 1$.
Similarly to \cite{zimmert2019connections}, define new loss estimates
\begin{equation} \label{eq:unbiased-estimator-2}
    \hat{\loss}_t(i) = \begin{cases}
        \frac{\loss_t(i)}{P_t(i)} \I\lcb{I_t \in \neigh{t}{i}} & \text{if $i \in V \setminus J_t$} \\[.5em]
        \frac{\loss_t(i)-1}{P_t(i)} \I\lcb{I_t \in \neigh{t}{i}} + 1 & \text{if $i \in J_t$}
    \end{cases}
\end{equation}
for which it still holds that $\E_t\bsb{\hat{\loss}_t} = \loss_t$ and that $\E_t\bsb{\hat{\loss}_t(i)^2} \le 1/P_t(i)$ for all $i \notin J_t$.
This change, along with the use of \Cref{lem:turan-tsallis} for the actions in $V \setminus S_t$, suffices in order to prove the following regret bound (see \Cref{app:main} for the proof) when the feedback graphs do not necessarily contain self-loops for all actions.

\begin{restatable}{rethm}{thmregret} \label{thm:regret}
    Let $G_1, \dots, G_T$ be a sequence of strongly observable undirected feedback graphs, where each $G_t$ has independence number $\alpha_t = \alpha$ for some common value $\alpha \in [K]$.
    If \Cref{alg:FTRL} is run with input
    \[
        \tparam = \frac12\bbrb{1 + \frac{\ln(K/\alpha)}{\sqrt{\ln(K/\alpha)^2+4} + 2}} \in [1/2,1)
        \qquad \text{ and } \qquad
        \eta = \frac13 \sqrt{\frac{2\tparam K^{1-\tparam}}{T(1-\tparam)\alpha^\tparam}} \enspace,
    \]
    and loss estimates~\eqref{eq:unbiased-estimator-2}, then its regret satisfies $R_T \le 6\sqrt{e\alpha T \lrb{2+\ln(K/\alpha)}}$.
\end{restatable}

\section{Adapting to Arbitrary Sequences of Graphs}
\label{s:doubling}
In the previous section, we assumed for simplicity that all the graphs have the same independence number. This independence number was then used to tune $\tparam$, the parameter of the Tsallis entropy regularizer used by the algorithm. In this section, we show how to extend our approach to the case when the independence numbers of the graphs are neither the same nor known a-priori by the learner. Had these independence numbers been known a-priori, one approach is to set $\tparam$ as in \Cref{thm:regret} but using the average independence number
\[\avgalpha = \frac{1}{T}\summ_{t=1}^T \alpha_t \enspace.\]
Doing so would allow us to achieve a $\calO\Bigl(\sqrt{\summ_{t=1}^T \alpha_t(1+\ln(K/\avgalpha))}\Bigr)$ regret bound. We now show that we can still recover a bound of the same order without prior knowledge of $\avgalpha$. 
For round $t$ and any fixed $\tparam \in [0,1]$, define
\[
    \tvarlps_t(q) = \sum_{i \in V \setminus S_t} \frac{p_t(i)^{2-\tparam}}{P_t(i)} \enspace.
\]
We know from \Cref{lem:turan-tsallis} that $\tvarlps_t(q)\leq \alpha_t^q$. Thus, we can leverage these observations and use a doubling trick (similar in principle to \cite{alon-journal}) to guess the value of $\avgalpha$. 
This approach is outlined in \Cref{alg:FTRL-doubling}. Starting with $r=0$ and $T_r=1$, the idea is to instantiate \Cref{alg:FTRL} at time-step $T_r$ with $\tparam$ and $\eta$ set as in \Cref{thm:regret} but with $2^r$ replacing the independence number. Then, at $t \geq T_r$, we increment $r$ and restart \Cref{alg:FTRL} only if
\begin{equation*}
    \frac{1}{T} \summ_{s=T_r}^t \tvarlps_s(\tparam_r)^{1/\tparam_r} > 2^{r+1},
\end{equation*}
since (again thanks to \Cref{lem:turan-tsallis}) the left-hand side of the above inequality is always upper bounded by $\avgalpha$. The following theorem shows that this approach essentially enjoys the same regret bound of \Cref{thm:regret} up to an additive $\log_2 \avgalpha$ term.

\begin{algorithm}
    \caption{$q$-FTRL for an arbitrary sequence of strongly observable undirected graphs} \label{alg:FTRL-doubling}
    \begin{algorithmic}
        \State \textbf{input:} Time horizon $T$
        \State \textbf{define:} For each $r \in \lcb{0, \dots, \floor{\log_2 K}}$,
        \[
        \tparam_r = \frac12 \bbrb{1 + \frac{\ln(K/2^r)}{\sqrt{\ln(K/2^r)^2+4} + 2}} \qquad \text{ and } \qquad \eta_r = \sqrt{\frac{2 \tparam_r K^{1-\tparam_r}}{11 T(1-\tparam_r)\lrb{2^r}^{\tparam_r}}}
        \]
        \State \textbf{initialization:} $T_0 \gets 1$, $r \gets 0$, instantiate \Cref{alg:FTRL} with $\tparam=\tparam_0$, $\eta=\eta_0$, and loss estimates \eqref{eq:unbiased-estimator-2}
        \For{$t = 1, \ldots, T$}
            \State Perform one step of the current instance of \Cref{alg:FTRL}
            \If{$\frac{1}{T} \sum_{s=T_r}^t \tvarlps_s(\tparam_r)^{1/\tparam_r} > 2^{r+1}$}
                \State $r \gets r+1$ 
                \State $T_r \gets t+1$
                \State Restart \Cref{alg:FTRL} with $\tparam=\tparam_r$, $\eta=\eta_r$, and loss estimates \eqref{eq:unbiased-estimator-2}
            \EndIf
        \EndFor
    \end{algorithmic}
\end{algorithm}

\begin{restatable}{rethm}{thmregretdoubling}\label{thm:regret-doubling}
Let $C = 4\sqrt{6 e}  \frac{\sqrt{\pi} + \sqrt{4-2\ln 2}}{\ln 2} $. Then, the regret of \Cref{alg:FTRL-doubling} satisfies
\[
    R_T \le C \sqrt{ \sum_{t=1}^T \alpha_t \bbrb{2 + \ln \bbrb{\frac{K} {\avgalpha}}} } + \log_2 \avgalpha \enspace.
\]
\end{restatable}
\begin{proof} [Proof sketch]
For simplicity, we sketch here the proof for the case when in every round $t$, all the nodes have self-loops; hence, $\tvarlps_t(q)=\tvar_t(q)$. See the full proof in \Cref{app:doubling}, which treats the general case in a similar manner. 
Let $n = \bce{\log_2 \avgalpha}$ and assume without loss of generality that $\avgalpha > 1$. Since \Cref{lem:turan-tsallis} implies that for any $r$ and $t$, $B_t(\tparam_r) \leq \alpha^{\tparam_r}_t$, we have as a consequence that for any $t \geq T_r$,
\[
    \frac{1}{T} \summ_{s=T_r}^t B_s(\tparam_r)^{1/\tparam_r} \leq \frac{1}{T} \summ_{s=T_r}^t \alpha_s \leq \avgalpha \leq 2^n \enspace. 
\]
Hence, the maximum value of $r$ that the algorithm can reach is $n-1$. In doing so, we will execute $n$ instances of \Cref{alg:FTRL}, each corresponding to a value of $r \in \{0,\dots,n-1\}$. 
For every such $r$, we upper bound the instantaneous regret at step $T_{r+1}-1$ (the step when the restarting condition is satisfied) by $1$, hence the added $\log_2 \avgalpha$ term in the regret bound. For the rest of the interval; namely, for $t \in [T_r,T_{r+1}-2]$, we have via \eqref{eq:FTRL-regret-with-variance-term} that the regret of \Cref{alg:FTRL} is bounded by
\begin{align} \label{eq:doubling-subregret-mainpaper}
    \frac{K^{1-q_r}}{\eta_r(1-q_r)} + \frac{\eta_r}{2q_r} \E \summ_{t=T_r}^{T_{r+1}-2}  B_t(q_r)\enspace.
\end{align}
Define $T_{r:r+1}=T_{r+1}-T_r-1$, and notice that
\begin{align*}
    \summ_{t=T_r}^{T_{r+1}-2}  B_t(\tparam_r)
    &\leq T_{r:r+1}   \bbrb{\frac{1}{T_{r:r+1}}  \summ_{t=T_r}^{T_{r+1}-2} B_t(\tparam_r)^{1/\tparam_r}}^{\tparam_r} \\
    &\leq T_{r:r+1}   \bbrb{\frac{T}{T_{r:r+1}}  2^{r+1}}^{\tparam_r} \leq 2 T \brb{2^r}^{\tparam_r} \enspace,
\end{align*}
where the first inequality follows due to Jensen's inequality since $\tparam_r \in (0,1)$, and the second follows from the restarting condition of \Cref{alg:FTRL-doubling}. After, plugging this back into \eqref{eq:doubling-subregret-mainpaper}, we can simply use the definitions of $\eta_r$ and $\tparam_r$ and bound the resulting expression in a similar manner to the proof of \Cref{thm:regret-self-loops}. Overall, we get that
\begin{align*}
    R_T &\leq  4\sqrt{3 e T} 
    \summ_{r=0}^{n-1} \sqrt{2^r \ln \brb{e^2 K 2^{-r}}} + \log_2 \avgalpha\enspace,
\end{align*}
from which the theorem follows by using \Cref{lem:doubling-sum} in \Cref{app:auxiliary}, which shows, roughly speaking, that the sum on the right-hand side is of the same order as its last term.
\end{proof}
Although \Cref{alg:FTRL-doubling} requires knowledge of the time horizon, this can be dealt with by applying a standard doubling trick on $T$ at the cost of a larger constant factor.
It is also noteworthy that the bound we obtained is of the form $\sqrt{T \avgalpha(1+\ln(K/\avgalpha))}$ and not $\sqrt{\summ_t \alpha_t(1+\ln(K/\alpha_t))}$. Although both coincide with the bound of \Cref{thm:regret} when $\alpha_t$ is the same for all time steps, the latter is smaller via the concavity of $x (1+\ln(K/x))$ in $x$.
It is not clear, however, whether there is a tuning of $q \in (0,1)$ that can achieve the second bound (even with prior knowledge of the entire sequence $\alpha_1, \dots, \alpha_T$ of independence numbers).

\section{An Improved Lower Bound via Multitask Learning} \label{s:lower-bound}
In this section we provide a new lower bound on the minimax regret showing that, apart from the bandits case, a logarithmic factor is indeed necessary in general. 
When the graph is fixed over time, it is known that a lower bound of order $\sqrt{\alpha T}$ holds for any value of $\alpha$ \cite{alon-journal,mannor2011side}. Whereas for the experts case ($\alpha=1$), the minimax regret is of order\footnote{As a lower bound, this is known to hold asymptotically as $K$ and $T$ grow. However, it can also be shown to hold non-asymptotically (though with worse leading constants); see \cite[Theorem 3.22]{Haussler-experts-lower-bound} or \cite[Theorem 3.6]{prediction-learning-games}.} $\sqrt{T \ln K}$ \cite{cesa1997use}. The following theorem provides, for the first time, a lower bound that interpolates between the two aforementioned bounds for the intermediate values of $\alpha$.
\begin{restatable}{rethm}{thmlowerbound} \label{thm:lower-bound}
    Pick any $K \ge 2$ and any $\alpha$
    such that $2 \leq \alpha \leq K$. Then, for any algorithm and for all $T \geq \frac{\alpha \log_{\alpha} K}{4\log(4/3)}$,
    there exists a sequence of losses and feedback graphs $G_1, \dots, G_T$ such that $\alpha(G_t) = \alpha$ for all $t=1, \dotsc, T$ and
    \begin{equation*}
       R_T \geq \frac{1}{18\sqrt{2}} \sqrt{\alpha T \log_{\alpha} K}.
    \end{equation*}
\end{restatable}
In essence, the proof of this theorem (see \Cref{app:lower-bound})
constructs a sequence of feedback graphs and losses that is equivalent to a hard instance of the multitask bandit problem (MTB) \cite{cesa2012combinatorial}, an important special case of combinatorial bandits with a convenient structure for proving lower bounds \cite{audibert2014combinatorial,cohen2017combinatorial,ito2019combinatorial}. We consider a variant of MTB in which, at the beginning of each round, the decision-maker selects an arm to play in each one of $M$ stochastic bandit games. Subsequently, the decision-maker only observes (and suffers) the loss of the arm played in a single randomly selected game. For proving the lower bound, we use a class of stationary stochastic adversaries (i.e., environments), each generating graphs and losses in a manner that simulates an MTB instance.

Fix $2 \le \alpha \le K = |V|$ and assume for simplicity that $\ngames= \log_{\alpha}K$ is an integer. 
We now construct an instance of online learning with time-varying feedback graphs $G_t = (V,E_t)$ with $\alpha(G_t) = \alpha$ that is equivalent to an MTB instance with $\ngames$ bandit games each containing $\alpha$ ``base actions''.
Since $K=\alpha^\ngames$, we can
uniquely identify each action in $V$ with a vector $a = \big(a(1),\ldots,a(\ngames)\big)$ in $[\alpha]^\ngames$.
The action $a_t \in V$ chosen by the learner at round $t$ is equivalent to a choice of base actions $a_t(1),\ldots,a_t(\ngames)$ in the $\ngames$ games.
The feedback graph at every round is sampled uniformly at random from a set of $\ngames$ undirected graphs $\{G^i\}_{i=1}^\ngames$, where $G^i=(V,E^i)$ is such that $(a,a')\in E^i$ if and only if $a(i)=a'(i)$. 
This means (see \Cref{f:three-graphs}) that each graph $G^i$ consists of $\alpha$ isolated cliques $\{C_{i,j}\}_{j=1}^\alpha$ such that an action $a$ belongs to clique $C_{i,j}$ if and only if $a(i)=j$. Clearly, the independence number of any such graph is $\alpha$.
Drawing feedback graph $G_t = G^i$ corresponds to the activation of game $i$ in the MTB instance.
Hence, choosing $a_t\in V$ with feedback graph $G_t=G^i$ is equivalent to playing base action $a_t(i)$ in game $i$ in the MTB.
As for the losses, we enforce that, given a feedback graph $G_t$, all actions that belong to the same clique of the feedback graph are assigned the same loss.
Namely, if $G_t=G^i$ and $a(i) = a'(i) = j$, then $\ell_t(a) = \ell_t(a')$, which can be seen as the loss $\loss_t(j)$ assigned to base action $j$ in game $G^i$.
To choose the distribution of the losses for the base actions, we apply the classic needle-in-a-haystack approach of \cite{auer1995gambling} over the $\ngames$ games.
More precisely, we construct a different environment for each action $a \in V$ in such a way that the distribution of the losses in each MTB game slightly favors (with a difference of a small $\epsilon>0$) the base action corresponding to $a$ in that game. 
The proof then proceeds similarly to, for example, the proof of Theorem 5 in \cite{audibert2014combinatorial} or Theorem 7 in \cite{eldowa2023information}.

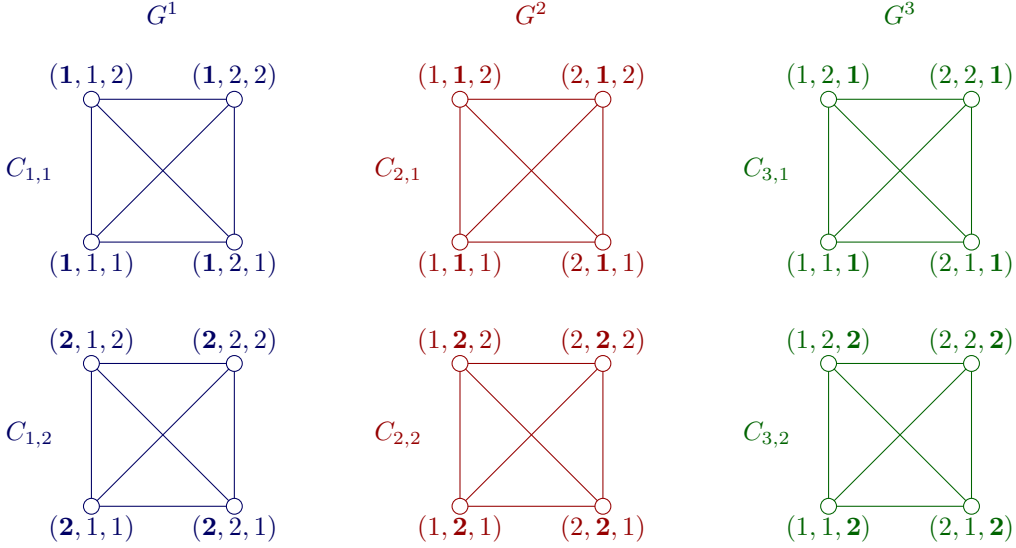
\begin{figure}
\centering
\newcommand{\xstretch}{1.9}
\newcommand{\xtranslation}{4.9}
\newcommand{\ystretch}{1.9}
\newcommand{\ytranslation}{1.85}
\newcommand{\positioning}{}
\newcommand{\mycolor}{}
    \definecolor{colone}{RGB}{0,0,100}
    \definecolor{coltwo}{RGB}{150,0,0} 
    \definecolor{colthree}{RGB} {0,100,0} 
\newcommand{\yo}{}
\newcommand{\nodesize}{3pt}

\begin{tikzpicture}

\foreach \g in {1,2,3} {
    \ifthenelse{\g = 1}{\renewcommand{\mycolor}{colone}}{}
    \ifthenelse{\g = 2}{\renewcommand{\mycolor}{coltwo}}{}
    \ifthenelse{\g = 3}{\renewcommand{\mycolor}{colthree}}{}
    \foreach \y in {0,1} {
        \draw[\mycolor] ({0*\xstretch+\g*\xtranslation}, {(0-\y*\ytranslation)*\ystretch}) 
           -- ({0*\xstretch+\g*\xtranslation}, {(1-\y*\ytranslation)*\ystretch});
        \draw[\mycolor] ({0*\xstretch+\g*\xtranslation}, {(0-\y*\ytranslation)*\ystretch}) 
           -- ({1*\xstretch+\g*\xtranslation}, {(1-\y*\ytranslation)*\ystretch});
        \draw[\mycolor] ({0*\xstretch+\g*\xtranslation}, {(0-\y*\ytranslation)*\ystretch}) 
           -- ({1*\xstretch+\g*\xtranslation}, {(0-\y*\ytranslation)*\ystretch});
        \draw[\mycolor] ({1*\xstretch+\g*\xtranslation}, {(0-\y*\ytranslation)*\ystretch}) 
           -- ({0*\xstretch+\g*\xtranslation}, {(1-\y*\ytranslation)*\ystretch});
        \draw[\mycolor] ({1*\xstretch+\g*\xtranslation}, {(0-\y*\ytranslation)*\ystretch}) 
           -- ({1*\xstretch+\g*\xtranslation}, {(1-\y*\ytranslation)*\ystretch});
        \draw[\mycolor] ({1*\xstretch+\g*\xtranslation}, {(1-\y*\ytranslation)*\ystretch}) 
           -- ({0*\xstretch+\g*\xtranslation}, {(1-\y*\ytranslation)*\ystretch});
    }
}

\foreach \g in {1,2,3} {
    \ifthenelse{\g = 1}{\renewcommand{\mycolor}{colone}}{}
    \ifthenelse{\g = 2}{\renewcommand{\mycolor}{coltwo}}{}
    \ifthenelse{\g = 3}{\renewcommand{\mycolor}{colthree}}{}
    \foreach \y in {0,1} {
        \foreach \i in {1,2} {
        \foreach \j in {1,2} {
            \ifthenelse{\g = 1}
            {
                \ifthenelse{\j=1}{\renewcommand{\positioning}{below}}{\renewcommand{\positioning}{above}}
                \draw[\mycolor] ({(\i-1)*\xstretch+\g*\xtranslation},{(\j-1-\y*\ytranslation)*\ystretch}) node[\positioning] {$(\boldsymbol{\ifthenelse{\y=0}{1}{2}},\i,\j)$};
            }{}
            \ifthenelse{\g = 2}
            {
                \ifthenelse{\j=1}{\renewcommand{\positioning}{below}}{\renewcommand{\positioning}{above}}
                \draw[\mycolor] ({(\i-1)*\xstretch+\g*\xtranslation},{(\j-1-\y*\ytranslation)*\ystretch}) node[\positioning] {$(\i,\boldsymbol{\ifthenelse{\y=0}{1}{2}},\j)$};
            }{}
            \ifthenelse{\g = 3}
            {
                \ifthenelse{\j=1}{\renewcommand{\positioning}{below}}{\renewcommand{\positioning}{above}}
                \draw[\mycolor] ({(\i-1)*\xstretch+\g*\xtranslation},{(\j-1-\y*\ytranslation)*\ystretch}) node[\positioning] {$(\i,\j,\boldsymbol{\ifthenelse{\y=0}{1}{2}})$};
            }{}

            \draw[\mycolor, fill=white] ({(\i-1)*\xstretch+\g*\xtranslation},{(\j-1-\y*\ytranslation)*\ystretch}) circle (\nodesize) node{};
        }
        }
        \ifthenelse{\y=0}{\renewcommand{\yo}{1}}{\renewcommand{\yo}{2}}
        \draw[\mycolor] ({(-0.2)*\xstretch+\g*\xtranslation},{(1-1-\y*\ytranslation)*\ystretch/2 + (2-1-\y*\ytranslation)*\ystretch/2}) node[left] {$C_{\g,\yo}$};
    }
    \draw[\mycolor] ({\xstretch/2+\g*\xtranslation}, {1.6*\ystretch}) node{$G^\g$};
}
\end{tikzpicture}
\caption{This figure shows an example of the multi-task bandit construction used to prove the lower bound. Here, $K=8$ and $\alpha=2$; thus, the number of games is $\ngames=3$. Each action is identified by a tuple of three numbers, each corresponding to a choice of one out of a pair of ``base actions'' in each game. Each of the three graphs in the figure corresponds to a game, such that two actions share an edge if and only if they choose the same base action in the corresponding game. At every round, a graph is randomly drawn, and all actions belonging to the same clique suffer the same loss.}
\label{f:three-graphs}
\end{figure}

While both our upper and lower bounds achieve the desired goal of interpolating between the minimax rates of experts and bandits, the logarithmic factors in the two bounds are not exactly matching. In particular, if we compare $1+\log_2(K/\alpha)$ and $\log_\alpha K$, we can see that although they coincide at $\alpha=2$ and $\alpha=K$, the former is larger for intermediate values.
It is reasonable to believe that the upper bound is of the correct order, seeing as it arose naturally as a result of choosing the best parameter for the Tsallis entropy regularizer, whereas achieving the extra logarithmic term in the lower bound required a somewhat contrived construction.

\section{Conclusions and Future Work}
In this work, we have shown that a proper tuning of the $q$-FTRL algorithm allows one to achieve a $\mathcal{O}\bigl(\sqrt{\alpha T(1+\ln(K/\alpha))}\bigr)$ regret for the problem of online learning with undirected strongly observable feedback graphs. Our bound interpolates between the minimax regret rates of the bandits and the experts problems, the two extremes of the strongly observable graph feedback spectrum. Furthermore, we have shown that an analogous bound can be achieved when the graphs vary over time, and without requiring any prior knowledge on the graphs. These results are complemented by our new lower bound of $\Omega\bigl(\sqrt{\alpha T (\ln K)/(\ln \alpha)}\bigr)$, which holds for $\alpha \geq 2$ and shows the necessity of a logarithmic factor in the minimax regret except for the bandits case. While our results provide the tightest characterization to date of the minimax rate for this setting, closing the small remaining gap (likely on the lower bound side) is an interesting problem.
After the submission of this manuscript, a subsequent work~\cite{chen2023interpolating} showed a lower bound for fixed feedback graphs composed of disjoint cliques that would imply worst-case optimality (up to constant factors) of our proposed algorithm for each pair of $K$ and~$\alpha$---see \Cref{app:comparison-concurrent} for a more detailed comparison with results therein.
Extending our results to the case of directed strongly observable feedback graphs is a considerably harder task---see \Cref{app:directed-setting} for a preliminary discussion.
Better understanding this more general setting is an interesting future direction.

\section*{Acknowledgements}
KE, EE, and NCB gratefully acknowledge the financial support from the MUR PRIN grant 2022EKNE5K (Learning in Markets and Society), funded by the NextGenerationEU program within the PNRR scheme (M4C2, investment 1.1), the FAIR (Future Artificial Intelligence Research) project, funded by the NextGenerationEU program within the PNRR-PE-AI scheme (M4C2, investment 1.3, line on Artificial Intelligence), and the EU Horizon CL4-2022-HUMAN-02 research and innovation action under grant agreement 101120237, project ELIAS (European Lighthouse of AI for Sustainability).
TC gratefully acknowledges the support of the University of Ottawa through grant GR002837 (Start-Up Funds) and that of the Natural Sciences and Engineering Research Council of Canada (NSERC) through grants RGPIN-2023-03688 (Discovery Grants Program) and DGECR-2023-00208 (Discovery Grants Program, DGECR - Discovery Launch Supplement).

\bibliographystyle{abbrv}
\DeclareRobustCommand{\VAN}[3]{#3}
\bibliography{main}
\DeclareRobustCommand{\VAN}[3]{#2}

\newpage
\appendix

\section{Auxiliary Results} \label{app:auxiliary}

\begin{lemma} \label{lem:FTRL-Tsallis-bandits-bound}
    If \Cref{alg:FTRL} is run with $q\in(0,1)$, learning rate $\eta > 0$, and non-negative loss estimates that satisfy $\E_t\bsb{\hat{\loss}_t} = \loss_t$ for all $t=1,\dots,T$, then its regret satisfies
    \begin{equation*}
        R_T \leq \frac{ K^{1-q}}{(1-q)\eta} + \frac{\eta}{2q} \sum_{t=1}^T \E \lsb{ \sum_{i \in V} p_t(i)^{2-q} \:\hat{\loss}_t(i)^2 } \enspace.
    \end{equation*}
\end{lemma}
\begin{proof}
    Let $i^* \in \argmin_{i \in V} \sum_{t=1}^T \loss_t(i)$ be an action that minimizes the cumulative loss, and let $\mathbf{e}_{i^*} \in \R^K$ be an indicator vector for $i^*$. Recall that for $t \in [T]$, $\E_t[\cdot] = \E[\cdot \mid I_1,\dots,I_{t-1}]$, and notice that $p_t$ is measurable with respect to the $\sigma$-algebra generated by $I_1,\dots,I_{t-1}$. Hence, using that 
    \[
        \E_t\bsb{\loss_t(I_t)} = \sum_{i \in V} p_t(i) \loss_t(i)
        \qquad \text{ and } \qquad
        \E_t\bsb{\hat{\loss}_t} = \loss_t \enspace,
    \]
    we have, via the tower rule and the linearity of expectation, that
    \[
        R_T = \E \lsb{ \sum_{t=1}^T \loss_t(I_t) } - \sum_{t=1}^T \loss_t(i^*)
        = \E \lsb{ \sum_{t=1}^T \langle p_t - \mathbf{e}_{i^*}, \loss_t \rangle}
        = \E \lsb{ \sum_{t=1}^T \langle p_t - \mathbf{e}_{i^*}, \hat{\loss}_t \rangle},
    \]
    from which we can obtain the desired result by using \Cref{lem:FTRL-Tsallis-bound} (which holds even if the loss $\hat{\loss}_t$ at each round $t \in [T]$ depends on the prediction $p_t$ made at that round).
\end{proof}

\begin{lemma} \label{lem:FTRL-Tsallis-bound}
    Let $q\in(0,1)$, $\eta > 0$, and $(y_t)_{t=1}^T$ be an arbitrary sequence of non-negative loss vectors in $\R^K$. Let $(p_t)_{t=1}^{T+1}$ be the predictions of FTRL with decision set $\Delta_K$ and the $q$-Tsallis regularizer~$\ftrlreg_{\tparam}$ over this sequence of losses. That is, $p_1 = \argmin_{p \in \Delta_{K}}  \ftrlreg_\tparam(p)$, and for $t \in [T]$,
    \[
     p_{t+1} = \argmin_{p \in \Delta_{K}} \eta \sum_{s=1}^{t} \ban{ y_s, p} + \ftrlreg_\tparam(p) \enspace.
    \]
    Then for any $u \in \Delta_{K}$,
    \begin{equation*}
        \sum_{t=1}^T \langle p_t - u, y_t \rangle \leq \frac{ K^{1-q}}{(1-q)\eta} + \frac{\eta}{2q} \sum_{t=1}^T  \sum_{i \in V} p_t(i)^{2-q} \:y_t(i)^2 \enspace.
    \end{equation*}
\end{lemma}
\begin{proof}
    By Theorem~28.5 in \cite{lattimore2020bandit}, we have that
    \begin{align*}
         \sum_{t=1}^T \langle p_t - u, y_t \rangle &\leq \frac{\psi_q(u)-\psi_q(p_1)}{\eta}+ \sum_{t=1}^T  \bbrb{  \langle p_t - p_{t+1}, y_t \rangle - \frac{1}{\eta} D_{\psi_q}(p_{t+1},p_t)} \\
        &=  \frac{ K^{1-q}- 1}{(1-q)\eta}+\sum_{t=1}^T  \bbrb{ \langle p_t - p_{t+1}, y_t \rangle - \frac{1}{\eta} D_{\psi_q}(p_{t+1},p_t)} \\
        &\leq \frac{ K^{1-q}}{(1-q)\eta}+\sum_{t=1}^T \bbrb{ \langle p_t - p_{t+1}, y_t \rangle - \frac{1}{\eta} D_{\psi_q}(p_{t+1},p_t)}\enspace,
    \end{align*}
    where $D_{\psi_q}(\cdot,\cdot)$ is the Bregman divergence based on $\psi_q$. For bounding each summand in the second term, we follow a similar argument to that used in Theorem 30.2 in \cite{lattimore2020bandit}. Namely, for each $i \in V$ and round $t \in [T]$, define $\overline{y}_t(i)=\I\{p_{t+1}(i) \leq p_t(i)\}y_t(i)$. We then have that
    \begin{align*}
        \langle p_t - p_{t+1}, &y_t \rangle - \frac{1}{\eta} D_{\psi_q}(p_{t+1},p_t)\\
        &\quad\leq \langle p_t - p_{t+1}, \overline{y}_t \rangle - \frac{1}{\eta} D_{\psi_q}(p_{t+1},p_t) \\
        &\quad= \frac{1}{\eta} \langle p_t - p_{t+1}, \eta \overline{y}_t \rangle - \frac{1}{2\eta} \bno{p_{t+1}-p_t}^2_{\nabla^2\psi_q(z_t)}\\
        &\quad\leq \frac{\eta}{2} \bno{\overline{y}_t}^2_{(\nabla^2\psi_q(z_t))^{-1}} \\
        &\quad= \frac{\eta}{2q} \sum_{i \in V} z_t(i)^{2-q}\: \overline{y}_t(i)^2 \\
        &\quad= \frac{\eta}{2q} \sum_{i \in V} \brb{\gamma_t p_{t+1}(i) + (1-\gamma_t)p_t(i)}^{2-q} \:\overline{y}_t(i)^2 \\
        &\quad\leq \frac{\eta}{2q} \sum_{i \in V} p_t(i)^{2-q} \:\overline{y}_t(i)^2 + \gamma_t \frac{\eta}{2q} \sum_{i \in V} \brb{p_{t+1}(i)^{2-q}-p_t(i)^{2-q}} \:\overline{y}_t(i)^2 \\
        &\quad\leq \frac{\eta}{2q} \sum_{i \in V} p_t(i)^{2-q} \:\overline{y}_t(i)^2 \\
        &\quad\leq \frac{\eta}{2q} \sum_{i \in V} p_t(i)^{2-q} \:y_t(i)^2 \enspace,
    \end{align*}
    where $z_t=\gamma_t p_{t+1} + (1-\gamma_t)p_t$ for some $\gamma_t \in [0,1]$; the first inequality holds due to the non-negativity of the losses, the second inequality is an application of the Fenchel-Young inequality, the second equality holds since the Hessian of $\psi_q$ is a diagonal matrix with $(\nabla^2\psi_q(x))_{i,i}=qx(i)^{q-2}$,
    the third inequality is an application of Jensen's inequality (since $q\in (0,1)$), and the fourth inequality holds since $\overline{y}_t(i)=0$ for any $i$ such that $p_{t+1}(i)^{2-q}>p_t(i)^{2-q}$.
\end{proof}

\begin{lemma} \label{lem:doubling-sum}
    Let $a$ and $b$ be positive integers such that $2 \leq a \leq b$, and let $n = \bce{\log_2 a}$. Then, 
    \begin{equation*}
        \summ_{r=0}^{n-1} \sqrt{2^r \ln \brb{e^2 b 2^{-r}}} \leq \frac{\sqrt{2\pi} + 2\sqrt{2-\ln 2}}{\ln 2} \sqrt{a \ln \bbrb{\frac{e^2 b}{a}}} \enspace.
    \end{equation*}
\end{lemma}
\begin{proof}
    Since $n \leq \log_2(2 b)$ and $2^r \ln \brb{e^2 b 2^{-r}}$ is monotonically increasing in $r$ for $r \in [0,\log_2(e b)]$, we can bound the sum by an integral:
    \begin{equation*}
        \summ_{r=0}^{n-1} \sqrt{2^r \ln \brb{e^2 b 2^{-r}}} \leq \int_{0}^n \sqrt{2^r \ln \brb{e^2 b 2^{-r}}} \diff r \enspace.
    \end{equation*}
    We proceed via a change of variable; let $x=e^2 b 2^{-r}$, and note that $\diff r = -\frac{\diff x}{x \ln 2}$. We then have that
    \begin{align*}
         \int_{0}^n \sqrt{2^r \ln \brb{e^2 b 2^{-r}}} \diff r &= \sqrt{e^2 b} \int_{0}^n \sqrt{\frac{2^r}{e^2 b} \ln \brb{e^2 b 2^{-r}}} \diff r \\
         &= -\frac{e \sqrt{b}}{\ln 2} \int_{e^2 b}^{e^2 b 2^{-n}} \sqrt{\frac{\ln x}{x^3} } \diff x = \frac{e \sqrt{b}}{\ln 2} \int_{e^2 b 2^{-n}}^{e^2 b } \sqrt{\frac{\ln x}{x^3} } \diff x \\
         &= \frac{e \sqrt{b}}{\ln 2} \Bsb{- \sqrt{2 \pi } \cdot \erfc\brb{\sqrt{(\ln x)/{2} }} -2 \sqrt{(\ln x)/{x} }}^{e^2 b}_{e^2 b 2^{-n}} \\
         &\leq \frac{e \sqrt{b}}{\ln 2} \lrb{\sqrt{2 \pi } \cdot \erfc\brb{\sqrt{\ln (e^2 b 2^{-n})/{2} }} + 2 \sqrt{\frac{2^{n} \ln (e^2 b 2^{-n})}{e^2 b} }}\enspace,
    \end{align*}
    where $\erfc(x) = 1-\frac{2}{\sqrt{\pi}} \int_{0}^x \exp(-z^2) \diff z$ is the complementary Gaussian error function, which is always positive. By \cite[Theorem~1]{cerf-bound}, we have that $\erfc(x) \leq \exp(-x^2)$. Consequently,
    \begin{align*}
        \int_{0}^n \sqrt{2^r \ln \brb{e^2 b 2^{-r}}} \diff r &\leq \frac{e \sqrt{b}}{\ln 2} \lrb{\sqrt{2 \pi } \sqrt{\frac{2^{n}}{e^2 b}}  + 2 \sqrt{\frac{2^{n} \ln (e^2 b 2^{-n})}{e^2 b} }} \\
        &=  \frac{\sqrt{2^{n}}}{\ln 2} \lrb{\sqrt{2 \pi } + 2 \sqrt{\ln (e^2 b 2^{-n})} } \\
        &\leq \frac{\sqrt{2a}}{\ln 2} \lrb{\sqrt{2 \pi } + 2 \sqrt{\ln \bbrb{\frac{e^2 b}{2a}}} } \\
        &\leq \frac{\sqrt{2\pi} + 2\sqrt{2-\ln 2}}{\ln 2} \sqrt{a \ln \bbrb{\frac{e^2 b}{a}}} \enspace,
    \end{align*}
    where in the second inequality we used once again the fact that $2^r \ln \brb{e^2 b 2^{-r}}$ is monotonically increasing in $r$ for $r \in [0,\log_2(e b)]$ to replace $n$ with $\log_2(a) + 1$, and the last inequality holds since $b/a \geq 1$.
\end{proof}

\section{Proofs of Section \ref{s:main}} \label{app:main}
In this section, we provide the proof of \Cref{thm:regret}, which is restated below.
\thmregret*
\begin{proof}
    Let $i^* \in \argmin_{i \in V} \sum_{t=1}^T \loss_t(i)$ and $\mathbf{e}_{i^*} \in \R^K$ be its indicator vector. 
    Whenever $J_t$ is nonempty, let $j_t \in V$ be the only action such that $J_t = \lcb{j_t}$.
    Similarly to \cite{zimmert2019connections}, let $z_t = \I\lcb{J_t \neq \emptyset} \I\lcb{I_t \in \neigh{t}{j_t}} \frac{1-\loss_t(j_t)}{1-p_t(j_t)}$ and define new losses $\tilde{\loss}_t(i) = \hat{\loss}_t(i) + z_t$ for each time step $t \in [T]$ and each action $i \in V$.
    Since $p_t, \mathbf{e}_{i^*} \in \Delta_K$, we have that $\langle p_t - \mathbf{e}_{i^*}, \hat{\loss}_t \rangle = \langle p_t - \mathbf{e}_{i^*}, \tilde{\loss}_t \rangle$ for every $t \in [T]$. Then, using the fact that $\E_t\bsb{\hat{\loss}_t} = \loss_t$, we get that
    \[
    R_T = \E \lsb{ \sum_{t=1}^T \langle p_t - \mathbf{e}_{i^*}, \hat{\loss}_t \rangle} = \E \lsb{ \sum_{t=1}^T \langle p_t - \mathbf{e}_{i^*}, \tilde{\loss}_t \rangle} \enspace,
    \]
    where the first equality holds via the same arguments made in the proof of \Cref{lem:FTRL-Tsallis-bandits-bound}.
    If we consider the optimization step of \Cref{alg:FTRL}, computing the same inner product over the new losses $\tilde{\loss}_1, \dots, \tilde{\loss}_T$ for some $p \in \Delta_K$ gives
    \[
        \Ban{\sum_{s=1}^t \tilde{\loss}_s, p} = \sum_{s=1}^t z_s + \Ban{\sum_{s=1}^t \hat{\loss}_s, p} \enspace,
    \]
    where the sum $\sum_{s=1}^t z_s$ is constant with respect to $p$.
    This implies that the objective functions in terms of either $(\hat{\loss}_t)_{t\in[T]}$ and $(\tilde{\loss}_t)_{t\in[T]}$, respectively, are minimized by the same probability distributions. However, notice that, unlike $(\hat{\loss}_t)_{t\in[T]}$, the loss vectors in $(\tilde{\loss}_t)_{t\in[T]}$ are always non-negative.
    Consequently, similar to the proof of \Cref{lem:FTRL-Tsallis-bandits-bound}, we may apply \Cref{lem:FTRL-Tsallis-bound} to upper bound the regret of \Cref{alg:FTRL} in terms of the losses $(\tilde{\loss}_t)_{t\in[T]}$.
    Doing so gives
    \begin{align} \label{eq:noloops-intermediate-regret-bound}
        \E \lsb{ \sum_{t=1}^T \langle p_t - \mathbf{e}_{i^*}, \tilde{\loss}_t \rangle}
        \le \frac{K^{1-q}}{\eta(1-q)} + \frac{\eta}{2q} \sum_{t=1}^T \E\Bbsb{\sum_{i\in V} p_t(i)^{2-q} \E_t\Bsb{\tilde{\loss}_t(i)^2}} \enspace.
    \end{align}
    We can bound the second term by observing that $\tilde{\loss}_t(j_t) = 1$ whenever $J_t \neq \emptyset$.
    Therefore,
    \begin{align*}
        \sum_{i\in V} p_t(i)^{2-q} \E_t\Bsb{\tilde{\loss}_t(i)^2}
        &\le 2\sum_{i \in V\setminus J_t} p_t(i)^{2-q}\E_t\Bsb{\hat{\loss}_t(i)^2} + 2\E_t\lsb{z_t^2} \sum_{i \in V\setminus J_t} p_t(i)^{2-q} + 1 \\
        &\le 2\sum_{i \in V\setminus J_t} \frac{p_t(i)^{2-q}}{P_t(i)} + 2\E_t\lsb{z_t^2} \sum_{i \in V\setminus J_t} p_t(i)^{2-q} + 1 \\
        &\le 2\sum_{i \in V\setminus J_t} \frac{p_t(i)^{2-q}}{P_t(i)} + 3 \enspace,
    \end{align*}
    where the second inequality holds because $\E_t\bsb{\hat{\loss}_t(i)^2} \le 1/P_t(i)$ for all $i \notin J_t$, and the third inequality follows from the fact that
    \[
        \E_t\lsb{z_t^2} \sum_{i \in V\setminus J_t} p_t(i)^{2-q} = \I\lcb{J_t \neq \emptyset}\frac{\brb{1-\loss_t(j_t)}^2}{1-p_t(j_t)} \sum_{i \in V\setminus J_t} p_t(i)^{2-q} \le 1 \enspace.
    \]
    We can handle the remaining sum by separating it over nodes $i\in S_t$, which satisfy $P_t(i) = 1-p_t(i)$ because of strong observability, and those in $\overline{S}_t = V\setminus S_t$.
    In the first case, any node $i \in S_t \setminus J_t$ has $p_t(i) \le 1/2$ and thus
    \[
        \sum_{i \in S_t\setminus J_t} \frac{p_t(i)^{2-q}}{P_t(i)}
        = \sum_{i \in S_t\setminus J_t} \frac{p_t(i)^{2-q}}{1-p_t(i)}
        \le 2\sum_{i \in S_t\setminus J_t} p_t(i)^{2-q} \le 2 \enspace.
    \]
    while in the second case we have that $\sum_{i \in \overline{S}_t} p_t(i)^{2-q}/P_t(i) \le \alpha^{q}$ by \Cref{lem:turan-tsallis} with $U = \overline{S}_t$ and $b=1-q$.
    Overall, we have shown that
    \begin{align} \label{eq:noloops-variance-term-bound}
        \sum_{i\in V} p_t(i)^{2-q} \E_t\Bsb{\tilde{\loss}_t(i)^2} \le 2 \sum_{i \in \overline{S}_t} \frac{p_t(i)^{2-q}}{P_t(i)} + 7 \le 2\alpha^{q} + 7 \le 9\alpha^{q} \enspace.
    \end{align}
    Plugging back into \eqref{eq:noloops-intermediate-regret-bound}, we obtain that
    \begin{align*} 
        R_T
        &\leq \frac{K^{1-q}}{\eta(1-q)} + \frac{9\eta}{2q} \alpha^{q}T \\ \nonumber
        &= 3\sqrt{\frac{2K^{1-q}\alpha^{q}}{q(1-q)} T} \\ \nonumber
        & \le 6\sqrt{e\alpha T \lrb{2+\ln(K/\alpha)}} \enspace,
    \end{align*}
    where the equality is due to our choice of $\eta$, and the last inequality follows as in the proof of \Cref{thm:regret-self-loops} together with our choice of $q$.
\end{proof}

\section{Proofs of Section \ref{s:doubling}} \label{app:doubling}
In this section, we provide the proof of \Cref{thm:regret-doubling}, which is restated below.
\thmregretdoubling*
\begin{proof}
Notice that if $\avgalpha = 1$, the initial guess is correct and the algorithm will never restart. Moreover, since in this case we have that $\alpha_t=1$ for all $t$, the theorem follows trivially from the regret bound of \Cref{thm:regret}. Hence, we can assume for what follows that $\avgalpha > 1$. Let $i^* \in \argmin_{i\in[K]} \sum_{t=1}^T \loss_t(i)$ and $n = \bce{\log_2 \avgalpha}$. Note that the maximum value of $r$ that the algorithm can reach is $n-1$. To see this, observe that \Cref{lem:turan-tsallis} implies that for any $r$ and $t$, $\tvarlps_t(\tparam_r) \leq \alpha^{\tparam_r}_t$. Consequently, for any $t \geq T_r$,
\[
    \frac{1}{T} \summ_{s=T_r}^t \tvarlps_s(\tparam_r)^{1/\tparam_r} \leq \frac{1}{T} \summ_{s=T_r}^t \alpha_s \leq \avgalpha \leq 2^n \enspace. 
\]
For $t \in [T]$, let $r_t$ be the value of $r$ at round $t$. 
Without loss of generality, we assume that $r$ takes each value in  $\{0,\dots,n-1\}$ for at least two rounds. Additionally, we define $T_{n}=T+2$ for convenience.
We start by decomposing the regret over the $n$ intervals (each corresponding to a value of $r$ in  $\{0,\dots,n-1\}$) and bounding the instantaneous regret with $1$ for each step in which we restart (i.e., at the last step of each but the last interval):
\begin{align}
    R_T &= 
    \E\Bbsb{\summ_{t=1}^T \brb{\loss_t(I_t) - \loss_t(i^*)}} \nonumber\\
    &\leq \E\Bbsb{\summ_{r=0}^{n-1} \summ_{t=T_r}^{T_{r+1}-2} \brb{\loss_t(I_t) - \loss_t(i^*)}} +  n - 1 \nonumber\\
    \label{doubling-regret-decomp}
    &\leq \E\Bbsb{\summ_{r=0}^{n-1} \summ_{t=T_r}^{T_{r+1}-2} \brb{\loss_t(I_t) - \loss_t(i^*)}} +  \log_2 \avgalpha \enspace.
\end{align}
For what follows, let $\mathbf{e}_{i^*} \in \R^K$ be an indicator vector for $i^*$ and let $\tilde{\loss}_t$ be as defined in the proof of \Cref{thm:regret}.
Fix $r \in \{0,\dots,n-1\}$, we proceed by bounding the regret in the interval $[T_r, T_{r+1}-2]$:
\begin{align}
     &\E\Bbsb{\summ_{t=T_r}^{T_{r+1}-2} \brb{\loss_t(I_t) - \loss_t(i^*)}} \nonumber\\
     &\quad= \E\Bbsb{\sum_{t=1}^T \I\bbcb{r_t = r, \frac{1}{T} \sum_{s=T_{r_t}}^t \tvarlps_s(\tparam_{r_t})^{1/\tparam_{r_t}} \leq 2^{r_t+1}} \brb{\loss_t(I_t)  -  \loss_t(i^*)}} \nonumber\\
     &\quad\stackrel{(a)}{=} \E\Bbsb{\sum_{t=1}^T \I\bbcb{r_t = r, \frac{1}{T} \sum_{s=T_{r_t}}^t \tvarlps_s(\tparam_{r_t})^{1/\tparam_{r_t}} \leq 2^{r_t+1}} \inprod{p_t - \mathbf{e}_{i^*}, \hat{\loss}_t}} \nonumber\\
     &\quad\stackrel{(b)}{=} \E\Bbsb{\sum_{t=1}^T \I\bbcb{r_t = r, \frac{1}{T} \sum_{s=T_{r_t}}^t \tvarlps_s(\tparam_{r_t})^{1/\tparam_{r_t}} \leq 2^{r_t+1}} \inprod{p_t - \mathbf{e}_{i^*}, \tilde{\loss}_t}} \nonumber\\
     &\quad= \E\Bbsb{\sum_{t=T_r}^{T_{r+1}-2} \inprod{p_t - \mathbf{e}_{i^*}, \tilde{\loss}_t}} \nonumber\\
     &\quad\stackrel{(c)}{\leq} \frac{K^{1-q_r}}{\eta_r(1-q_r)} + \frac{\eta_r}{2q_r} \E\Bbsb{\sum_{t=T_r}^{T_{r+1}-2} \sum_{i\in V} p_t(i)^{2-q_r} \tilde{\loss}_t(i)^2} \nonumber\\
     &\quad\stackrel{(d)}{=} \frac{K^{1-q_r}}{\eta_r(1-q_r)} + \frac{\eta_r}{2q_r} \E\Bbsb{\sum_{t=1}^T \I\bbcb{r_t = r, \frac{1}{T} \sum_{s=T_{r_t}}^t \tvarlps_s(\tparam_{r_t})^{1/\tparam_{r_t}} \leq 2^{r_t+1}} \E_t\bbsb{\sum_{i\in V} p_t(i)^{2-q_r} \tilde{\loss}_t(i)^2}} \nonumber\\
     &\quad\stackrel{(e)}{\leq} \frac{K^{1-q_r}}{\eta_r(1-q_r)} + \frac{\eta_r}{2q_r} \E\Bbsb{\sum_{t=1}^T \I\bbcb{r_t = r, \frac{1}{T} \sum_{s=T_{r_t}}^t \tvarlps_s(\tparam_{r_t})^{1/\tparam_{r_t}} \leq 2^{r_t+1}} (2\tvarlps_t(q_r) + 7)} \nonumber\\ \label{doubling-subregret}
     &\quad= \frac{K^{1-q_r}}{\eta_r(1-q_r)} + \frac{\eta_r}{2q_r} \E\Bbsb{\sum_{t=T_r}^{T_{r+1}-2} (2\tvarlps_t(q_r) + 7)} \enspace,
\end{align}
where $(a)$ follows since 
$\E_t\bsb{\loss_t(I_t)} = \sum_{i \in V} p_t(i) \loss_t(i)$, $\E_t\bsb{\hat{\loss}_t} = \loss_t$, and the indicator at round~$t$ is measurable with respect to $\sigma(I_1,\dots,I_{t-1})$, that is, the $\sigma$-algebra generated by $I_1,\dots,I_{t-1}$;
$(b)$~follows since $\langle p_t - \mathbf{e}_{i^*}, \hat{\loss}_t \rangle = \langle p_t - \mathbf{e}_{i^*}, \tilde{\loss}_t \rangle$ holds by the definition of $\tilde{\loss}_t$;
$(c)$ is an application of \Cref{lem:FTRL-Tsallis-bound}, justifiable with the same argument leading to \eqref{eq:noloops-intermediate-regret-bound} in the proof of \Cref{thm:regret};
$(d)$ uses once again that the indicator at round $t$ is measurable with respect to $\sigma(I_1,\dots,I_{t-1})$;
finally, $(e)$ follows via \eqref{eq:noloops-variance-term-bound}.
Define $T_{r:r+1}=T_{r+1}-T_r-1$, and notice that
\begin{align*}
    \summ_{t=T_r}^{T_{r+1}-2}  \tvarlps_t(\tparam_r) &= \frac{T_{r:r+1}}{T_{r:r+1}}  \summ_{t=T_r}^{T_{r+1}-2}  \brb{\tvarlps_t(\tparam_r)^{1/\tparam_r}}^{\tparam_r} \\
    &\leq T_{r:r+1}   \bbrb{\frac{1}{T_{r:r+1}}  \summ_{t=T_r}^{T_{r+1}-2} \tvarlps_t(\tparam_r)^{1/\tparam_r}}^{\tparam_r} \\
    &\leq T_{r:r+1}   \bbrb{\frac{T}{T_{r:r+1}}  2^{r+1}}^{\tparam_r} \\
    &\leq 2 T \brb{2^r}^{\tparam_r} \enspace,
\end{align*}
where the first inequality follows due to Jensen's inequality since $\tparam_r \in (0,1)$, 
and the second follows from the restarting condition of \Cref{alg:FTRL-doubling}. 
Next, we plug this inequality back into \eqref{doubling-subregret}, and then, similar to the proof of \Cref{thm:regret}, we use the definitions of $\eta_r$ and $\tparam_r$ and bound the resulting expression to get that
\begin{align*}
    \E\Bbsb{\summ_{t=T_r}^{T_{r+1}-2} \brb{\loss_t(I_t)  -  \loss_t(i^*)}}
    &\leq \frac{K^{1-q_r}}{\eta_r(1-\tparam_r)} + \frac{11\eta_r}{2\tparam_r} T \lrb{2^r}^{\tparam_r} \\
    &\leq 2\sqrt{11 e T 2^r \Brb{2 + \ln \brb{K 2^{-r}}}} \leq 4\sqrt{3 e T 2^r \ln \brb{e^2 K 2^{-r}}} \enspace. 
\end{align*}
We then sum this quantity over $r$ and use \Cref{lem:doubling-sum} to get that
\begin{align*}
    \E\Bbsb{\summ_{r=0}^{n-1} \summ_{t=T_r}^{T_{r+1}-2} \brb{\loss_t(I_t)  - \loss_t(i^*)}}
    &\leq 4\sqrt{3 e T} \summ_{r=0}^{n-1} \sqrt{2^r \ln \brb{e^2 K 2^{-r}}} \\
    &\leq 4\sqrt{6 e}  \frac{\sqrt{\pi} + \sqrt{4-2\ln 2}}{\ln 2} \sqrt{\avgalpha T  \brb{2 + \ln \lrb{K / \avgalpha}}} \enspace,
\end{align*}
which, together with \eqref{doubling-regret-decomp}, concludes the proof.
\end{proof}

\section{Proof of the Lower Bound} \label{app:lower-bound}
In this section, we prove the lower bound provided in \Cref{s:lower-bound}, which we restate below. As remarked before, our proof makes use of known techniques for proving lower bounds for the multitask bandit problem. In particular, parts of the proof are adapted from the proof of Theorem 7 in \cite{eldowa2023information}. 
\thmlowerbound*
\begin{proof}
Once again, we define $M = \log_\alpha K$, which we assume for now to be an integer; we discuss in the end how to extend the proof to the case when it is not. The proof will be divided into five parts I--V.
We begin by formalizing the class of environments described in \Cref{s:lower-bound} and stating two useful lemmas.

\subsection*{I. Preliminaries}
We remind the reader that we identify each action in $V$ with a vector $a = \big(a(1),\ldots,a(\ngames)\big) \in [\alpha]^\ngames$. We will focus on a set of $\ngames$ undirected graphs $\G = \{G^i\}_{i=1}^\ngames$, where 
$G^i$ consists of $\alpha$ isolated cliques (with self-loops) $\{C_{i,j}\}_{j=1}^\alpha$ such that an action $a$ belongs to clique $C_{i,j}$ if and only if $a(i)=j$. As remarked before, all these graphs have independence number $\alpha$. For convenience, we also use actions in $V$ as functions from $\G$ to $[\alpha]$, with $a(G^i) = a(i)$. 

An environment is identified by a function $\mu\colon [\alpha] \times \G \to [0,1]$ such that at every round $t$, after having drawn a graph $G_t$ from the uniform distribution over $\G$ (denoted with $U_\G$), the environment latently draws for each $j \in [\alpha]$ and $G \in \G$, a Bernoulli random variable $\gamma_t(j;G)$ with mean $\mu(j;G_t)$. Subsequently, for defining the loss of action $a \in V$ at round $t$, we simply set $\loss_t(a) = \gamma_t(a(G_t);G_t)$, whose expectation, conditioned on $G_t$, is $\mu(a(G_t);G_t)$.
To simplify the notation, we use $\mu(a;G)$ as shorthand for $\mu(a(G);G)$ and $\gamma_t(a;G)$ as shorthand for $\gamma_t(a(G);G)$. 
Denote by $A_t$ the action picked by the player at round $t$, which is chosen prior to observing $G_t$. We will focus on the following notion of stochastic regret, which we define for environment $\mu$ as:
\begin{equation*}
    \storeg(\mu) = \max_{a \in V} \E_\mu \biggl[\sum_{t=1}^T (\loss_t(A_t) - \loss_t(a)) \biggr] \enspace,
\end{equation*}
where $\E_\mu[\cdot]$ denotes the expectation with respect to the sequence of losses and graphs generated by environment $\mu$, as well as the randomness in the choices of the player. We can use the tower rule to rewrite this expression as
\begin{align}
    \storeg(\mu) &= \max_{a \in V} \sum_{t=1}^T \E_\mu \biggl[\E_\mu \biggl[\E_\mu \biggl[  \ell_t(A_t) - \ell_t(a) \,\bigg|\, G_t,A_t\biggr] \,\bigg|\, A_t\biggr]\biggr] \nonumber\\
    &= \max_{a \in V} \sum_{t=1}^T \E_\mu \biggl[\E_\mu \biggl[\mu(A_t;G_t) - \mu(a;G_t) \,\bigg|\, A_t\biggr]\biggr] \nonumber\\
    &= \max_{a \in V} \sum_{t=1}^T \E_\mu \biggl[\sum_{i=1}^M U_\G(G^i)(\mu(A_t;G^i) - \mu(a;G^i))\biggr] \nonumber\\ \label{low:varying:stoch-reg}
    &=\max_{a \in V} \frac{1}{M} \sum_{i=1}^M \E_\mu \biggl[\sum_{t=1}^T (\mu(A_t;G^i) - \mu(a;G^i)) \biggr] \enspace.
\end{align}
For a fixed algorithm, one can show via standard arguments that 
\[
\sup_{(\loss_t)_{t=1}^T,(G_t)_{t=1}^T} R_T \geq \sup_\mu \storeg(\mu) \enspace.
\] 
Hence, it suffices for our purposes to prove a lower bound for the right-hand side of this inequality. 

In the following, we will have to be more precise about the probability measure with respect to which the expectation in \eqref{low:varying:stoch-reg} is defined. Let $\obs_t \in \{0,1\}^{K/\alpha}$ denote the vector of losses observed by the player at round $t$, which corresponds to the losses of the actions connected to $A_t$ assuming that a systematic ordering of the actions makes it clear which coordinate of $\obs_t$ belongs to which action. Let $\ones$  and $\zeros$ be the $K/\alpha$ dimensional\footnote{Note that $K/\alpha = \alpha^{\ngames-1}$ is an integer since $M (\geq 1)$ was assumed to be an integer.} vectors of all ones and all zeros respectively. Clearly, we have that $\obs_t = \gamma_t(A_t;G_t) \ones = \loss_t(A_t) \ones$, which is a binary random variable taking values in $\{\zeros, \ones\}$. Let $\obspr_\mu$ be the probability distribution of $\obs_t$ in environment $\mu$. Notice then that we have that
\begin{equation} \label{low:varying:observation-vector}
    \obspr_\mu(\gamma_t = \ones \,|\, G_t=G, A_t=a) = \mu(a; G) \enspace.
\end{equation}

Let $H_t = (A_1,G_1,\obs_1, \dotsc, A_t, G_t, \obs_t) \in (V \times \G \times \{0,1\}^{K/\alpha})^t$ be the interaction trajectory after $t$ steps. The policy $\pi$ adopted by the player can be modelled as a sequence of probability kernels $\{\pi_t\}_{t=1}^T$ each mapping the trajectory so far to a distribution over the actions, i.e., $A_t$ is sampled from $\pi_t(\cdot \,|\, H_{t-1})$. An environment $\mu$ and a policy $\pi$ (implicit in the notation, and fixed throughout the rest of the proof) together define a distribution $\pr_\mu$ over the set of possible trajectories of $T$ steps such that:
\begin{align*}
    \pr_\mu(H_T) &= \prod_{t=1}^T \pi_t(A_t \,|\, H_{t-1}) U_\G(G_t) \obspr_\mu(\lambda_t \,|\, G_t,A_t) \enspace.
\end{align*}

If $P$ and $Q$ are two distributions defined on the same space, let $\kl{P}{Q}$ and $\tv(P,Q)$ be the KL-divergence and the total variation distance respectively between $P$ and $Q$. Furthermore, let $\klber{p}{q}$ be the KL-divergence between two Bernoulli random variables with means $p$ and $q$.
The following lemma provides an expression for the KL-divergence between two the probability distributions associated to two environments.
\begin{lemma} \label{low:varying:kl-decomp}
For a fixed policy, let $\mu$ and $\mu'$ be two environments as described above. Then,
\begin{equation*}
    \kl{\pr_\mu}{\pr_{\mu'}} = \frac{1}{M} \sum_{i=1}^M \sum_{a\in V} N_\mu(a;T) \bklber{\mu(a;G^i)}{\mu'(a;G^i)} \enspace,
\end{equation*}
where $N_\mu(a;T) = \E_\mu \Bsb{\sum_{t=1}^T \I\{A_t=a\}}$.
\end{lemma}
\begin{proof}
The proof is similar to that of Lemma 15.1 in \cite{lattimore2020bandit}. Namely, we have in our case that
\begin{align*}
    \kl{\pr_\mu}{\pr_{\mu'}} &= \E_\mu \biggl[\ln\frac{\pr_\mu(H_T)}{\pr_{\mu'}(H_T)}\biggr] \\
    &= \E_\mu \biggl[\ln\frac{\prod_{t=1}^T \pi_t(A_t \,|\, H_{t-1}) U_\G(G_t) \obspr_\mu(\lambda_t \,|\, G_t,A_t)}{\prod_{t=1}^T \pi_t(A_t \,|\, H_{t-1}) U_\G(G_t) \obspr_{\mu'}(\lambda_t \,|\, G_t,A_t)}\biggr] \\
    &= \sum_{t=1}^T \E_\mu \biggl[ \ln \frac{\obspr_{\mu}(\lambda_t \,|\, G_t,A_t)}{\obspr_{\mu'}(\lambda_t \,|\, G_t,A_t)}\biggr] \\
    &= \sum_{t=1}^T \E_\mu \biggl[\E_\mu \biggl[\E_\mu \biggl[ \ln \frac{\obspr_{\mu}(\lambda_t \,|\, G_t,A_t)}{\obspr_{\mu'}(\lambda_t \,|\, G_t,A_t)} \,\bigg|\, G_t,A_t\biggr] \,\bigg|\, A_t\biggr]\biggr] \\
    &= \sum_{t=1}^T \E_\mu \biggl[\E_\mu \biggl[\kl{\obspr_\mu(\cdot \,|\, G_t,A_t)}{\obspr_{\mu'}(\cdot \,|\, G_t,A_t)} \bigg|A_t\biggr]\biggr] \\
    &= \sum_{t=1}^T \E_\mu \biggl[\sum_{i=1}^M U_\G(G^i) \kl{\obspr_\mu(\cdot \,|\, G^i,A_t)}{\obspr_{\mu'}(\cdot \,|\, G^i,A_t)} \biggr] \\
    &= \frac{1}{M} \sum_{i=1}^M \sum_{t=1}^T \E_\mu \biggl[ \kl{\obspr_\mu(\cdot \,|\, G^i,A_t)}{\obspr_{\mu'}(\cdot \,|\, G^i,A_t)} \biggr] \\
    &= \frac{1}{M} \sum_{i=1}^M \sum_{a\in V} N_\mu(a;T) \kl{\obspr_\mu(\cdot \,|\, G^i,a)}{\obspr_{\mu'}(\cdot \,|\, G^i,a)} \\
    &= \frac{1}{M} \sum_{i=1}^M \sum_{a\in V} N_\mu(a;T) \bklber{\mu(a;G^i)}{\mu'(a;G^i)} \enspace,
\end{align*}
where the last equality holds via \eqref{low:varying:observation-vector}.
\end{proof}
The following standard lemma, adapted from \cite{lattimore2020bandit}, will be used in the sequel.
\begin{lemma} \label{low:conversion-lemma}
Let P and Q be probability measures on the same measurable space $(\Omega, \mathcal{F})$. Let $a < b$ and $X: \Omega \xrightarrow{} [a,b]$ be an $\mathcal{F}$-measurable random variable. Then,
\begin{equation*}
    \left|\int_\Omega X(\omega) dP(\omega) - \int_\Omega X(\omega) dQ(\omega)\right| \leq (b-a) \sqrt{\frac{1}{2} \kl{P}{Q}} \enspace.
\end{equation*}
\end{lemma} 
\begin{proof}
We have, by Exercise 14.4 in \cite{lattimore2020bandit}, that
\begin{equation*}
    \left|\int_\Omega X(\omega) dP(\omega) - \int_\Omega X(\omega) dQ(\omega)\right| \leq (b-a) \tv(P,Q) \enspace,
\end{equation*}
from which the lemma follows by applying Pinsker's inequality.
\end{proof}

\subsection*{II. Choosing the environments}

    We will construct a collection of environments $\{\mu_a\}_{a \in V}$, each associated to an action, such that for any $i \in [\ngames]$ and $j \in [\alpha]$,
    \[
    \mu_a(j;G^i) = \frac12 - \epsilon \I\{a(i)=j\} \enspace,
    \]
    where $0<\epsilon\leq\frac{1}{4}$ will be tuned later. In words, for a fixed graph, environment $\mu_a$ gives a slight advantage to actions that are connected to $a$ in that graph, and thus agree with $a$ in the corresponding game. 
    Additionally, for every $a \in V$ and $i \in [M]$, we define environment $\mu^{-i}_a$ to be such that for any $s \in [\ngames]$ and $j \in [\alpha]$,
    \begin{equation*}
        \mu^{-i}_a(j;G^s) = 
        \begin{cases}
        \frac12, & \text{if } s=i \\
        \mu_a(j;G^s), & \text{otherwise.}
        \end{cases}
    \end{equation*}
    Similar to \cite{eldowa2023information}, we will define, for every $i \in [M]$, an equivalence relation $\sim_i$ on the arms such that
    \begin{equation*}
        a \sim_i a' \iff \forall s \in [M]\setminus\{i\}, a'(s)=a(s) \enspace,
    \end{equation*}
    for any $a, a' \in V$.
    This means that two arms are equivalent according to $\sim_i$ if and only if their choices of base actions coincide in all games that are different from $i$. 
    Let $V / \sim_i$ be the set of equivalence classes of $\sim_i$. It is easy to see that $V / \sim_i$ contains exactly $\alpha^{M-1}$ equivalence classes, and that each class consists of $\alpha$ actions, each corresponding to a different choice of base action in game $i$. 
    Notice then that for an equivalence class $W \in V / \sim_i$, all environments $\mu^{-i}_{a}$ with $a \in V$ are indeed identical. In the sequel, this environment will also be referred to as $\mu_{W}^{-i}$.

\subsection*{III. Lower-bounding the regret of a single environment}

    Note that in environment $\mu_a$, we have that $a = \argmin_{a'\in V} \sum_{i=1}^M \mu_a(a';G^i)$. Consequently, starting from \eqref{low:varying:stoch-reg} we get that 
    \begin{align*}
        \storeg(\mu_a) &= \sum_{i=1}^M \frac{1}{M} \E_{\mu_a} \biggl[\sum_{t=1}^T (\mu_a(A_t;G^i) - \mu_a(a;G^i)) \biggr] \\
        &= \sum_{i=1}^M \frac{1}{M} \E_{\mu_a} \biggl[\sum_{t=1}^T \left(\frac{1}{2} - \epsilon \I\{A_t(i)=a(i)\} - \left(\frac{1}{2}-\epsilon\right)\right) \biggr] \\
        &= \frac{\epsilon}{M} \sum_{i=1}^M \E_{\mu_a} \biggl[\sum_{t=1}^T (1 - \I\{A_t(i)=a(i)\}) \biggr] \\
        &= \frac{\epsilon}{M} \sum_{i=1}^M \brb{T - N_{\mu_a}(i,a;T)} \enspace,
    \end{align*}
    where for environment $\mu$, action $a$, and game $i$, $N_{\mu}(i,a; T) = \E_{\mu} \bsb{\sum_{t=1}^T  \mathbb{I}\{A_t(i) = a(i)\}}$ is the expected number of times in environment $\mu$ that the action chosen by the policy agrees with action $a$ in game $i$.
    Next, we use Lemma \ref{low:conversion-lemma} to obtain that
    \begin{equation} \label{low:varying:reg-kl}
        \storeg(\mu_a) \geq \frac{\epsilon}{M} \sum_{i=1}^M \bigg(T - N_{\mu^{-i}_a}(i,a;T)-T\sqrt{\frac{1}{2}\bkl{\pr_{\mu^{-i}_a}}{\pr_{{\mu_a}}} }\bigg) \enspace.
    \end{equation}

    For bounding the KL-divergence term, we start from Lemma \ref{low:varying:kl-decomp}:
    \begin{align*}
         \bkl{\pr_{\mu^{-i}_a}}{\pr_{{\mu_a}}}
         &= \frac{1}{M} \sum_{s=1}^M \sum_{a'\in V} N_{\mu^{-i}_a}(a';T) \bklber{\mu^{-i}_a(a';G^s)}{\mu_a(a';G^s)} \\
         &= \frac{1}{M} \sum_{a'\in V} N_{\mu^{-i}_a}(a';T) \bklber{\mu^{-i}_a(a';G^i)}{\mu_a(a';G^i)} \\
         &= \frac{1}{M} \sum_{a'\in V} N_{\mu^{-i}_a}(a';T) \bklber{1/2}{1/2 - \epsilon \I\{a'(i)=a(i)\}} \\
         &= \frac{1}{M} \sum_{a'\in V}  \I\{a'(i)=a(i)\} N_{\mu^{-i}_a}(a';T) \bklber{1/2}{1/2 - \epsilon} \\
         &\leq \frac{c \epsilon^2}{M} \sum_{a'\in V}  \I\{a'(i)=a(i)\} N_{\mu^{-i}_a}(a';T) \\
         &= \frac{c \epsilon^2}{M} \sum_{a'\in V}  \I\{a'(i)=a(i)\} \E_{\mu^{-i}_a} \sum_{t=1}^T \I\{A_t=a'\} \\
         &= \frac{c \epsilon^2}{M} \E_{\mu^{-i}_a}\sum_{t=1}^T \I\{A_t(i) = a(i)\} \\
         &= \frac{c \epsilon^2}{M} N_{\mu^{-i}_a}(i,a;T) \enspace,
    \end{align*}
    where the second equality holds since the two environments only differ in $G^i$, and the inequality holds for $\epsilon \leq \frac{1}{4}$ with $c = 8\log{\frac{4}{3}}$. Plugging back into \eqref{low:varying:reg-kl} gets us that
    \begin{equation} \label{low:varying:reg-one-env}
        \storeg(\mu_a) \geq \frac{\epsilon}{M} \sum_{i=1}^M \bigg(T - N_{\mu^{-i}_a}(i,a;T)-\epsilon T\sqrt{\frac{c}{2M}N_{\mu^{-i}_a}(i,a;T)}\bigg) \enspace.
    \end{equation}

\subsection*{IV. Summing up}
    Fix $i \in [M]$. Notice that for $W \in V/\sim_i$,
    \[
        \sum_{a \in W} \I\{A_t(i) = a(i)\} = 1 
    \]
    since each action in $W$ corresponds to a different choice of base action in game $i$. 
    Hence,
    \begin{align*}
        \frac{1}{\alpha^{M}} \sum_{a \in V} N_{\mu^{-i}_a}(i,a;T) &= \frac{1}{\alpha^{M}} \sum_{W \in V/\sim_i}
        \sum_{a \in W} N_{\mu^{-i}_a}(i,a;T) \\
        &= \frac{1}{\alpha^{M}} \sum_{W \in V/\sim_i}
        \sum_{a \in W} N_{\mu^{-i}_W}(i,a;T) \\
        &= \frac{1}{\alpha^{M}} \sum_{W \in V/\sim_i}
         \E_{\mu^{-i}_W} \lsb{\sum_{t=1}^T \sum_{a \in W} \I\{A_t(i) = a(i)\}} \\
         &= \frac{1}{\alpha^{M}} \alpha^{M-1} T =  \frac{T}{\alpha} \enspace.
    \end{align*}
    Using this together with \eqref{low:varying:reg-one-env} allows us to conclude that
    \begin{align*}
        \sup_\mu \storeg(\mu) 
        &\geq \frac{1}{\alpha^{M}} \sum_{a \in V}  \storeg(\mu_a) \\
        &\geq \frac{1}{\alpha^{M}} \sum_{a \in V}  \frac{\epsilon}{M} \sum_{i=1}^M \bigg(T - N_{\mu^{-i}_a}(i,a;T)-\epsilon T\sqrt{\frac{c}{2M}N_{\mu^{-i}_a}(i,a;T)}\bigg) \\
        &\geq \frac{\epsilon}{M} \sum_{i=1}^M \biggl(T - \frac{1}{\alpha^{M}} \sum_{a \in V}  N_{\mu^{-i}_a}(i,a;T)-\epsilon T \sqrt{\frac{c}{2 M \alpha^{M}}  \sum_{a \in V} N_{\mu^{-i}_a}(i,a;T)}\biggr) \\
        &= \frac{\epsilon}{M} \sum_{i=1}^M \bigg(T - \frac{T}{\alpha}-\epsilon T \sqrt{\frac{cT}{2 M \alpha}}\bigg) \\
        &= \epsilon T \bigg(1 - \frac{1}{\alpha}-\epsilon \sqrt{\frac{cT}{2 M \alpha}}\bigg) \\
        &\geq \epsilon T \bigg(\frac{1}{2}-\epsilon \sqrt{\frac{cT}{2 M \alpha}}\bigg) \enspace ,
    \end{align*}
    where the third inequality holds due to the concavity of the square root, and the last inequality holds by our assumption that $\alpha \geq 2$.
    Setting $\epsilon = \frac{1}{4} \sqrt{\frac{2 M \alpha}{cT}}$ yields that
    \begin{align*}
        \sup_\mu \storeg(\mu) &\geq \frac{1}{16}\sqrt{\frac{2}{c}}\cdot \sqrt{\alpha T M}
        \geq \frac{1}{18} \sqrt{\alpha T M}
        = \frac{1}{18} \sqrt{\alpha T \log_{\alpha} K} \enspace ,
    \end{align*}
    whereas it holds that $\epsilon\leq\frac{1}{4}$ thanks to the assumption made on $T$ in the statement of the theorem.

\subsection*{V. The case when \texorpdfstring{$\log_\alpha K$}{M} is not an integer}
If $M$ is not an integer,\footnote{Note that $M$ is never smaller than $1$ since $\alpha \leq K$.} we can use the same construction as before for the first $\alpha^{\floor{M}}$ actions and force the remaining actions to behave identically to some action in the construction.
That is, we can designate a certain action such that, in all environments, all the excess actions receive the same loss as this action and are connected to it, to each other, and to every action that happens to share an edge with this designated action in a given graph (in other words, we are expanding the designated action into a clique). 
This way, the independence number of all the graphs in the construction is still $\alpha$, and the excess actions do not provide any extra utility to the learner; playing one of them is exactly like playing the designated action, and the construction does not hide this from the player. 
We can then obtain the same bound as before but in terms of $\floor{M}$, thus costing us an extra $1/\sqrt{2}$ factor to recover the desired bound (using that $\floor{M} \geq M/2$).
\end{proof}

\section{Comparison with \mytcite{chen2023interpolating}} \label{app:comparison-concurrent}
In \cite{chen2023interpolating}, the authors consider a special case of the undirected feedback graph problem where the graph (fixed and known) is composed of $\alpha$ disjoint cliques with self-loops. For $j \in [\alpha]$, let $m_j$ denote the number of actions in the $j$-th clique, implying that $\sum_{j=1}^\alpha m_j = K$ (the number of arms). For this problem, \cite[Theorem 4]{chen2023interpolating} provides a lower bound of order $\sqrt{T \sum_{j=1}^\alpha \ln(m_j + 1)}$. In particular, if the cliques are balanced (i.e., $m_1 = \dots = m_\alpha = K/\alpha$), the lower bound becomes of order $\sqrt{\alpha T \ln(1+K/\alpha)}$, thus matching the regret bound of \Cref{alg:FTRL}.
This means that, for any value of $1\le\alpha\le K$, there are feedback graphs on $K$ nodes with independence number $\alpha$ such that no other algorithm can achieve a better minimax regret guarantee than that of our proposed algorithm.

We emphasize that this does not imply \emph{graph-specific} minimax optimality.
Indeed, as shown in~\cite{chen2023interpolating}, when the cliques are unbalanced, the regret guarantee of our algorithm can be inferior to that of the algorithm they proposed, which matches the $\sqrt{T \sum_{j=1}^\alpha \ln(m_j + 1)}$ bound.
However, beyond the disjoint cliques case, their algorithm requires computing a minimum clique cover for the given feedback graph $G$, which is known to be NP-hard~\cite{karp1972}.
More importantly, their reliance on a clique cover leads to a dependence of the regret on the clique cover number $\theta(G)$ instead of the independence number $\alpha(G)$.
One can argue that the ratio between $\theta(G)$ and $\alpha(G)$ can be $\Omega(K/(\ln K)^2)$ for most graphs on a sufficiently large number $K$ of vertices (e.g., see~\cite[Section~6]{mannor2011side}).
Finally, it is not clear how to generalize their approach to time-varying feedback graphs (informed or uninformed).
Hence, despite the contributions of our work and those of \cite{chen2023interpolating}, the problem of characterizing the minimax regret rate at a graph-based granularity still calls for further investigation.

\section{Directed Strongly Observable Feedback Graphs} \label{app:directed-setting}

In this section, we consider the case of directed strongly observable graphs. 
For a directed graph $G=(V,E)$, let $\inneigh{G}{i} = \{j \in V \,:\, (j,i) \in E\}$ be the in-neighbourhood of node $i \in V$ in $G$, and let $\outneigh{G}{i} = \{j \in V \,:\, (i,j) \in E\}$ be its out-neighbourhood. A directed graph $G$ is strongly observable if for every $i \in V$, at least one of the following holds: $i \in \inneigh{G}{i}$ or $j \in \inneigh{G}{i}$ for all $j \neq i$.
The independence number $\alpha(G)$ is still defined in the same manner as before; that is, the cardinality of the largest set of nodes such that no two nodes share an edge, regardless of orientation.
The interaction protocol is the same as in the undirected case, except that, in each round $t \in [T]$, the learner only observes the losses of the actions in $\outneigh{G_t}{I_t}$, which is the out-neighbourhood in graph $G_t$ of the action $I_t$ picked by the learner. As before, we will use $\inneigh{t}{i}$ and $\outneigh{t}{i}$ to denote $\inneigh{G_t}{i}$ and $\outneigh{G_t}{i}$ respectively.
For this setting, a bound of $\calO\brb{\sqrt{\alpha T} \cdot \ln(KT)}$ was proven in~\cite{alon2015beyond} for the $\textsc{Exp3.G}$ algorithm. Later, \cite{zimmert2019connections} proved a bound of $\mathcal{O}\big(\sqrt{\alpha T(\ln K)^3}\big)$ for $\textsc{OSMD}$ with a variant of the $q$-Tsallis entropy regularizer where $q$ was chosen as $1-1/(\ln K)$.  

To use \Cref{alg:FTRL} in the directed case, one can define loss estimates analogous to \eqref{eq:unbiased-estimator-2} by using the in-neighbourhood in place of the neighbourhood in the relevant quantities. Namely, let $S_t = \lcb{i \in V : i \notin \inneigh{t}{i}}$, $J_t = \lrc{i \in S_t : p_t(i) > 1/2}$, and $P_t(i) = \sum_{j \in \inneigh{t}{i}} p_t(j)$. The loss estimates (again due to \cite{zimmert2019connections}) can then be given by
\begin{equation*}
    \hat{\loss}_t(i) = \begin{cases}
        \frac{\loss_t(i)}{P_t(i)} \I\lcb{I_t \in \inneigh{t}{i}} & \text{if $i \in V \setminus J_t$} \\[.5em]
        \frac{\loss_t(i)-1}{P_t(i)} \I\lcb{I_t \in \inneigh{t}{i}} + 1 & \text{if $i \in J_t$} \enspace.
    \end{cases}
\end{equation*}
\Cref{alg:FTRL} with these loss estimates can be analyzed in a similar manner to the proof of \Cref{thm:regret}, with the major difference being the way that the variance term is handled for actions with self-loops. Namely, the relevant term is
\[
    \sum_{i \in V: i\in \inneigh{t}{i}} \frac{p_t(i)^{2-q}}{\sum_{j\in \inneigh{t}{i}} p_t(j)},
\]
on which we elaborate more in the following.

Let $p \in \Delta_K$ and $\beta \in (0, 1/2)$ be such that $\min_{i \in V} p(i) \ge \beta$.
We first consider the variance term given by the negative Shannon entropy regularizer.
It is known~\cite{alon2015beyond} that such a variance term, restricted to nodes with a self-loop in the strongly observable feedback graph $G=(V,E)$, has an upper bound of the form
\begin{equation} \label{eq:variance-bound-shannon-entropy}
    \sum_{i \in V: i\in \inneigh{G}{i}} \frac{p(i)}{\sum_{j\in \inneigh{G}{i}} p(j)} \le 4\alpha(G) \ln \bbrb{\frac{4K}{\alpha(G) \beta}} \enspace.
\end{equation}
In addition to the fact that this variance bound has a linear dependence on the independence number $\alpha(G)$ of $G$, we observe that there is a logarithmic factor in $K/\alpha$ and $1/\beta$ given by the fact that we now consider directed graphs.
The main problem is that, in general, we cannot hope to improve upon the above logarithmic factor as it can be shown to be unavoidable unless we manage to restrict the probability distributions we consider.
Indeed, it is possible to show~\cite[Fact~4]{alon-journal} that there exist probability distributions $p\in \Delta_K$ and directed strongly observable graphs $G$ for which $\alpha(G)=1$ and
\[
    \sum_{i \in V: i\in \inneigh{G}{i}} \frac{p(i)}{\sum_{j\in \inneigh{G}{i}} p(j)} = \frac{K+1}{2} = \frac12 \log_2\bbrb{\frac{4}{\min_{i} p(i)}}
    = \alpha(G) \log^{\omega(1)}\bbrb{\frac{K}{\alpha(G)}}
    \enspace.
\]
A usual way to avoid this is to introduce some explicit exploration to the probability distributions in order to force a lower bound on the probabilities of all nodes, e.g., as in $\textsc{Exp3.G}$~\cite{alon2015beyond}.
This would bring the linear dependence on $K$ down to $\alpha$ in the above bad case, while, on the other hand, introducing a $\ln(KT)$ factor which then worsens the overall dependence on the time horizon $T$.

Consider now the variance term given by the analysis of the $q$-FTRL algorithm.
As already argued in \Cref{s:main}, we can reuse the variance bound in~\eqref{eq:variance-bound-shannon-entropy} for the case of negative Shannon entropy because
\[
    \sum_{i \in V: i\in \inneigh{G}{i}} \frac{p(i)^{2-q}}{\sum_{j\in \inneigh{G}{i}} p(j)} \le \sum_{i \in V: i\in \inneigh{G}{i}} \frac{p(i)}{\sum_{j\in \inneigh{G}{i}} p(j)}
\]
for any $q \in (0,1)$, and such a bound is the best known so far for the general case of directed strongly observable graphs.
However, we can be more clever in the way we utilize it.
Similarly to the proof of \cite[Theorem~14]{zimmert2019connections}, we can gain an advantage from the adoption of $q$-FTRL by splitting the sum in the variance term into two sums according to some adequately chosen threshold $\beta$ on the probabilities of the individual nodes.
More precisely, by choosing $\beta \approx \exp\brb{-\ln(K/\alpha)\ln K}$ and $q=1-1/(\ln K)$, we can prove that
\[
    \sum_{i \in V: i\in \inneigh{G}{i}} \frac{p(i)^{2-q}}{\sum_{j\in \inneigh{G}{i}} p(j)} = \calO\bbrb{\alpha \ln K \bbrb{1+\ln\frac{K}{\alpha}}} \enspace.
\]
We can further argue that, by following a similar analysis as in the proofs of Theorems~\ref{thm:regret-self-loops} and~\ref{thm:regret}, this variance bound would allow to show that the regret of $q$-FTRL is $\calO\Brb{\sqrt{\alpha T \brb{1+\ln(K/\alpha)}} \cdot \ln K}$, where there is an additional $\ln K$ factor when compared to our regret bound in the undirected case (\Cref{thm:regret}).

The presence of extra logarithmic factors is to be expected in the directed case, as many edges between distinct nodes might reduce the independence number of the graph, while providing information in one direction only.
However, the undirected graph $G'$ obtained from any directed strongly observable graph $G$ by reciprocating edges between distinct nodes has the same independence number $\alpha(G')=\alpha(G)$ but the regret guarantee given by the more general analysis of $q$-FTRL would introduce a spurious $\ln K$ multiplicative factor.
We remark that all the currently available upper bounds on the variance term (either with negative Shannon entropy or negative $q$-Tsallis entropy regularizers) do not exactly reflect the phenomenon of a gradually disappearing logarithmic factor when the graph is closer to being undirected (i.e., has fewer unreciprocated edges).

Taking these observations into account, we believe that it should be possible to achieve tighter guarantees that match our intuition, by improving the currently available tools.
The bound on the variance term, for instance, is one part of the analysis that might be improvable.
We might want to have a similar bound as~\eqref{eq:variance-bound-shannon-entropy} but with a sublinear dependence on $\alpha$ that varies according to the parameter $q$ of the negative $q$-Tsallis entropy; e.g., ignoring logarithmic factors, we could expect it to become of order $\alpha^q$ as we managed to prove for the undirected case (\Cref{lem:turan-tsallis}).
Doing so could allow a better tuning of $q$ that might lead to improved logarithmic factors in the regret.

\end{document}